\documentclass{article}

\usepackage{arxiv}

\usepackage[utf8]{inputenc} 
\usepackage[T1]{fontenc}    
\usepackage{hyperref}       
\usepackage{url}            
\usepackage{booktabs}       
\usepackage{amsfonts}       
\usepackage{nicefrac}       
\usepackage{microtype}      
\usepackage{xcolor}         

\usepackage{bm}
\usepackage{bbm}
\usepackage{comment}         
\usepackage{multicol}
\usepackage{multirow}
\usepackage{caption}
\usepackage{natbib}

\newcommand{\set}[1]{\mathcal{#1}}

\newcommand\Pbb {\mathbb{P}}

\usepackage{amsmath,amssymb,amsthm}

\theoremstyle{plain}

\newtheorem{theorem}{Theorem}[section]
\newtheorem{lemma}[theorem]{Lemma}

\newtheorem{corollary}[theorem]{Corollary}
\newtheorem{proposition}[theorem]{Proposition}

\newtheorem{definition}[theorem]{Definition}
\newtheorem{assumption}[theorem]{Assumption}

\newtheorem*{remark}{Remark}

\usepackage{color}

\usepackage{graphicx}

\newcommand{\R}{\mathbb{R}}

\newcommand{\Ebb}{\mathbb{E}}

\newcommand{\KL}{\mathrm{KL}}

\newcommand{\Gap}{\Delta}

\usepackage{times}

\newcommand{\pa}{\mathrm{\pa}}

\newcommand{\Ebbnu}{\mathbb{E}_{\nu}}
\newcommand{\Ftminone}{\mathcal{F}_{t-1}}

\newcommand{\RN}[1]{%
  \textup{\uppercase\expandafter{\romannumeral#1}}%
}

\newcommand*\interior[1]{#1^{\mathsf{o}}}

\usepackage{algorithm}
\usepackage{algorithmic}

\newcommand*{\email}[1]{\texttt{#1}}

\usepackage{xcolor}
\newcount\Comments  
\newcommand{\kibitz}[2]{\ifnum\Comments=1\textcolor{#1}{#2}\fi}
\newcommand{\mkato}[1]{\kibitz{black}{#1}}

\usepackage{authblk}

\allowdisplaybreaks

\title{Optimal Best Arm Identification\\ in Two-Armed Bandits with a Fixed Budget\\ under a Small Gap}
\renewcommand{\shorttitle}{}

\author[1]{Masahiro Kato\thanks{\email{masahiro\_kato@cyberagent.co.jp}.}$\ \ $}
\author[1,2]{Kaito Ariu}
\author[3]{Masaaki Imaizumi}
\author[1]{Masahiro Nomura}
\author[5]{Chao Qin}

\affil[1]{AI Lab, CyberAgent, Inc.}
\affil[2]{School of Electrical Engineering and Computer Science, KTH}
\affil[3]{Department of Basic Science / Komaba Institute for Science, the University of Tokyo}
\affil[5]{Columbia Business School, Columbia University}

\begin{document}

\maketitle

\begin{abstract}
We consider fixed-budget best-arm identification in two-armed Gaussian bandit problems. One of the longstanding open questions is the existence of an optimal strategy under which the probability of misidentification matches a lower bound. We show that a strategy following the Neyman allocation rule \citep{Neyman1934OnTT} is asymptotically optimal when the gap between the expected rewards is small. First, we review a lower bound derived by \citet{Kaufman2016complexity}. Then, we propose the ``Neyman Allocation (NA)-Augmented Inverse Probability weighting (AIPW)'' strategy, which consists of the sampling rule using the Neyman allocation with an estimated standard deviation and the recommendation rule using an AIPW estimator. Our proposed strategy is optimal because the upper bound matches the lower bound when the budget goes to infinity and the gap goes to zero. 
\end{abstract}

\section{Introduction}
We study the \emph{best arm identification (BAI) with a fixed budget} in stochastic two-armed Gaussian bandit problems. The goal is to identify an arm that has the highest expected reward with the smallest failure probability under a fixed number of rounds, called a \emph{budget} \citep{Bubeck2009,Audibert2010}. 
Formally, we consider the following setting given a fixed budget $T$: 
for $t\in [T] := \{1,2,\dots, T\}$ and $a \in \{1,0\}$, there exists a potential independent reward $X_{a, t} \in\mathbb{R}$ from a distribution $\nu_a$ with an expected reward $\mu_a$.
Let $\mathcal{F}_{t}$ be the sigma-algebra by all observations up to time $t$.

A tuple of the distributions ${\nu} = (\nu_1, \nu_0)$ is called a \emph{bandit model}. In this study, we consider Gaussian bandit models defined as $\mathcal{M} = \{\nu = (\mathcal{N}(\mu_1, \sigma^2_1), \mathcal{N}(\mu_0, \sigma^2_0)): (\mu_1, \mu_0)\in\mathbb{R}^2, \mu_1\neq \mu_0\}$, where $\mathcal{N}(\mu_a, \sigma^2_a)$ is a Gaussian distribution with the mean $\mu_a$ and variance $\sigma_a^2$ for $a \in \{1,0\}$.
The best arm $a^*(\nu) \in \{1,0\}$ is defined as $\mu_{a^*} > \max_{a \neq a^*(\nu)} \mu_a$, which is assumed to exist uniquely. 
Without loss of generality, we assume that $\mu_1 > \mu_0$, hence $a^*(\nu) = 1$.
At time $t$, an agent randomly samples an arm $A_t \in 
\{1,0\}$ and receives a reward $X_t=\sum_{a\in\{1, 0\}}\mathbbm{1}[A_t = a]X_{a, t}$.

A \emph{strategy} for the BAI problem determines which arms to sample and which arm to choose as the best arm.
Formally, we define a strategy as a pair $((A_t)_{t\in[T]}, \hat{a}_T)$, where (i) $(A_t)_{t\in[T]}$ is a \textit{sampling rule} that determines which arm $A_t$ is chosen in each $t$ based on $\mathcal{F}_{t-1}$, and (ii) $ \hat{a}_T$ is a \textit{recommendation rule} that samples an arm  $\hat{a}_T$ based on $\mathcal{F}_T$. A sampling rule is often designed based on a \textit{target allocation ratio}; that is, the ratio of the expected number of samples of an arm. A target allocation ratio determines the number of samples allocated to each arm, and estimating the ratio during the round plays an important role in the strategy.

Our goal is to find a strategy that minimizes the probability of misidentification $\mathbb{P}_{\nu}( \hat{a}_T \neq a^*(\nu))$, where $\mathbb{P}_{\nu}$ (resp. $\mathbb{E}_{\nu}$) is the probability law (resp.
expectation) of a sequence of random variables $(X_t)_{t\in[T]}$.
We call a strategy is \emph{consistent}, if $\mathbb{P}_{\nu}( \hat{a}_T \neq a^*(\nu)) \to 0$ as $T \to \infty$.
To evaluate the performance of strategies, we focus on the following metric for $\mathbb{P}_{\nu}( \hat{a}_T \neq a^*(\nu))$ used in many studies, such as \citet{Kaufman2016complexity}:
\begin{align*}
    -\frac{1}{T}\log \mathbb{P}_{\nu}( \hat{a}_T \neq a^*(\nu)).
\end{align*}
We remark that the upper bound (resp. lower bound) of this term works as a lower bound (resp. upper bound) of the probability $\mathbb{P}_{\nu}( \hat{a}_T \neq a^*(\nu))$ since $x \mapsto -\log x$ is a strictly decreasing function. 

One of the main interests in this field is to derive a tight lower bound for the probability of misidentification and an optimal strategy under which the probability of misidentification matches the lower bound. 
\citet{glynn2004large} proposes an optimal strategy when an optimally selected target allocation ratio is given. Their target allocation ratio is identical to the one given in the Neyman allocation \citep{Neyman1934OnTT}, where samples are allocated to each treatment arm based on the ratio of its standard deviation.
However, they assume that the standard deviations is known and do not consider the issue of estimating them. \citet{Kaufman2016complexity} derives a distribution-dependent lower bound for the misidentification probability, which is agnostic to the optimal target allocation ratio; that is, their lower bound presumes the knowledge of the standard deviations.  
Despite the seminal result, it is still unknown whether there exists a strategy under which the probability of misidentification matches the lower bound of \citet{Kaufman2016complexity} when the standard deviations are unknown.
\citep{kaufmann2020hdr}. 

The main contribution of this study is the development of a strategy under which the probability of misidentification matches the distribution-dependent lower bound of  \citet{Kaufman2016complexity}. 
First, we propose the \textit{Neyman Allocation-Augmented Inverse Probability Weight} (NA-AIPW) strategy consisting of the following two components.
(i) a sampling rule with the Neyman Allocation (NA), and
(ii) a recommendation rule using the Augmented Inverse Probability Weighting (AIPW) estimator \citep{Robins1994,bang2005drestimation}. 
We show that the NA-AIPW strategy is asymptotically optimal under a \emph{small gap}. When $T \to \infty$ and a gap $\Delta = \mu_1 - \mu_0$ between the expected rewards goes to zero, the upper bound of the NA-AIPW strategy matches the lower bound of \citet{Kaufman2016complexity}. In Table \ref{table:contribution}, we summarize our contributions.

\begin{table}[t]
 \caption{Comparison of the contributions.}
    \label{table:contribution}
    \centering
    \scalebox{0.8}[0.8]{
    \begin{tabular}{|l|c|c|}
    \hline
        & Lower bound  & Upper bound \\
    \hline
       \citet{glynn2004large} & - &  (with a known standard deviation) \\
    \hline
        \citet{Kaufman2016complexity} & $\checkmark$ &  (with a known standard deviation) \\
    \hline
        Ours & $\checkmark$ (the lower bound of \citet{Kaufman2016complexity}) & $\checkmark$ (small gap) \\
    \hline
    \end{tabular}
    }
    \vspace{-0.5cm}
\end{table}

Investigating optimality under the small gap corresponds to evaluating one of the worst-case performances. 
A small gap implies a situation in which it is difficult to identify the best arm. Our proposed strategy is asymptotically optimal in the sense that it matches the lower bound in the small gap asymptotically. Therefore, our result can be interpreted as a worst-case optimality of the NA-AIPW strategy for the probability of misidentification. From the technical viewpoint, one of the reasons why the probability of misidentification matches the lower bound is partly because the estimation error of the optimal target allocation ratio becomes negligible with respect to the probability of misidentification.
In BAI with fixed confidence, lil'UCB \citep{Jamieson2014} considers optimality under a similar situation.

To evaluate the probability of misidentification under a small gap, we modified a large deviation bound for martingales shown by \citet{Fan2013,fan2014generalization}. 
Then, we apply the large deviation bounds to obtain the upper bound with the AIPW estimator whose asymptotic variance matches the lower bound for the asymptotic variance. 

\textbf{Organization.} The remainder of this paper is organized as follows.
In Section~\ref{sec:lowerbound}, we describe the lower bound, following \citet{Kaufman2016complexity}. In Section~\ref{sec:track_aipw}, we define the NA-AIPW strategy. In Section~\ref{sec:asymp_opt}, we present our main theorem on the asymptotic optimality of our proposed strategy with its proof.
In Section~\ref{sec:experiments}, we provide the empirical results of the numerical experiments to demonstrate the efficiency of our strategy. In Section~\ref{sec:related}, we discuss the related work and remaining problems. A detailed discussion of the AIPW estimator can be found in Sections \ref{sec:track_aipw} and \ref{sec:aipw_est} and Appendix~\ref{appdx:proof_large_deviation}.f

\section{Lower bound}
\label{sec:lowerbound}
As a preparation, we introduce a distribution-dependent lower bound for the probability of misidentification in BAI with a fixed budget. For BAI, \citet{Kaufman2016complexity} presents the distribution-dependent lower bounds specific to each reward distribution. Their lower bounds correspond to those for regret minimization, popularized by \citet{Lai1985}. Based on their results, we provide lower bounds for Gaussian bandit models $\mathcal{M}$. 

In order to present lower bounds and for subsequent theoretical analysis, we assume the following conditions on the bandit models.
\begin{assumption} \label{asm:bounded_mean_variance}
There exist known constants $C_{\mu}, C_{\sigma^2} > 0$ such that, for all $a \in \{1,0\}$, $| \mu_a| \le  C_\mu$ and $\max\{ 1/\sigma^2_a, \sigma^2_a\} < C_{\sigma^2}$. 
\end{assumption}

Recall that $\Delta = \mu_1 - \mu_0$ is a gap of the means between $\nu_1$ and $\nu_0$ in $\nu$. Then, we obtain the following statement from the arguments in Section~5.1 in \citet{Kaufman2016complexity}. 
\begin{proposition}
\label{prp:lowerbound_2arms}
Under Assumption~\ref{asm:bounded_mean_variance}, for each $\nu \in \mathcal{M}$, any consistent strategy satisfies 
\begin{align*}
    \limsup_{T \to \infty} - \frac{1}{T}\log \mathbb{P}_\nu(\hat{a}_T \neq a^*(\nu)) \le \frac{\Delta^2}{2 \big(\sqrt{\sigma^2_1} + \sqrt{\sigma^2_0}\big)^2}.
\end{align*}
\end{proposition}
For the sake of completeness, we show it in Appendix~\ref{appdx:prp:lowerbound_2arms}. 

It is known that there exist optimal strategies under which the probability of misidentification asymptotically matches the lower bound when the standard deviations are known \citep{glynn2004large,Kaufman2016complexity}. In existing optimal strategies using the known standard deviations, such as \citet{Kaufman2016complexity}, we sample each arm with the ratio of the standard deviations as well as the Neyman allocation. However, to the best of our knowledge, if standard deviations need to be estimated by adaptive experimentation, no optimal strategy has been proposed. Furthermore, \citet{Carpentier2016} derives another lower bound, which implies that there is no optimal strategies for the lower bound in Proposition~\ref{prp:lowerbound_2arms} when the gap $\Delta$ is fixed. In this study, we consider a strategy under which the probability of misidentification matches the lower bound in Proposition~\ref{prp:lowerbound_2arms} under the small gap, $\Delta \to 0$. This result implies optimality in a worst-case. 

\begin{remark}[Lower bound for two-armed Bernoulli bandits]
\citet{Kaufman2016complexity} also presents a lower bound for two-armed Bernoulli bandits, where each reward follows Bernoulli distributions. They also discuss that for Bernoulli bandits, an optimal strategy can be approximated by a strategy that uses uniform sampling and recommends an arm with the highest sample average of observed rewards.
\end{remark}

\section{Proposed strategy: the NA-AIPW strategy}
\label{sec:track_aipw}
In this section, we define our strategy, which consists of sampling and recommendation rules. 
We define a truncation operator: for a variable $v \in \R$ and a constant $c \geq 1$, $\mathrm{thre}(v;c) := \max\{ \min\{v,c\},1/c\}$.

\subsection{NA-AIPW strategy}
First, we define a target allocation ratio, which determines the ratio of the expected number of samples of an arm that the sampling rule follows. Then, we define an optimal target allocation ratio as
\begin{align*}
w^*_1 = \frac{\sqrt{\sigma^2_1}}{\sqrt{\sigma^2_1} + \sqrt{\sigma^2_0}},\mbox{~~and~~} w^*_0 = 1 - w^*_1.
\end{align*}
A sampling rule following this allocation ratio is known as the Neyman allocation rule \citep{Neyman1934OnTT}. \citet{glynn2004large} and \citet{Kaufman2016complexity} also propose this allocation. 
This optimal target allocation ratio is unknown when the variances are unknown; therefore, to use this ratio, we need to estimate it from observations during the bandit process. 

In each $t \in [T]$, our sampling rule randomly samples an arm, following an estimated target allocation ratio; that is, we sample an arm with the probability identical to the estimated target allocation ratio (target allocation probability). In final time $T$, our recommendation rule recommends an arm with the highest estimated expected reward. 
Based on these rules, we refer to this as the NA-AIPW strategy.\footnote{Similar strategies are often used in the context of the average treatment effect estimation by an adaptive experiment \citep{Laan2008TheCA,Kato2020adaptive}.} 

\subsection{Sampling rule with NA rule using an estimated standard deviation}
We provide a sampling strategy referred to as a \textit{Neyman allocation} (NA) strategy. For $a\in\{1,0\}$ and $t\in[T]$, let $w_{a,t}$ be an estimated target allocation ratio such that $\forall a\in\{1,0\}, \ w_{a, t} > 0$ and $\sum_{a\in\{1,0\}} {w}_{a,t} = 1$. In each time $t$, we obtain $\gamma_t$ from the uniform distribution on $[0,1]$ and sample arm $A_t = 1$ if $\gamma_t \leq w_{1, t}$ and $A_t = 0$ if $\gamma_t > w_{1, t}$. 

As an initialization, we sample arm $1$ at time $1$ and arm $0$ at time $2$ and set $w_{a,1} = w_{a,2} = 1/2$ for $a\in\{1,0\}$.
In time $t > 2$, for all $a\in\{1,0\}$, we estimate the optimal target allocation ratio $w^*_a$ using past observations $\mathcal{F}_{t-1}$.
To construct $w_{a,t}$, we define a estimator of the expected reward as $\tilde{\mu}_{a, t} = ({\sum^{t-1}_{s=1}\mathbbm{1}[A_s = a]})^{-1}\sum^{t-1}_{s=1}\mathbbm{1}[A_s = a] X_{a, s}$, a second moment estimator $\tilde{\zeta}_{a, t} = ({\sum^{t-1}_{s=1}\mathbbm{1}[A_s = a]})^{-1}\sum^{t-1}_{s=1}\mathbbm{1}[A_s = a] X^2_{a, s}$, and a variance estimator $\tilde{\sigma}^2_{a,t} = \tilde{\zeta}_{a, t} - \left(\tilde{\mu}_{a, t}\right)^2$ for $t \geq 3$. 
For $t=1,2$, we set $\tilde{\mu}_{a, t} = \tilde{\zeta}_{a, t} = 0$.
Then, we estimate the variance $\sigma^2_a$ for $a\in\{1,0\}$ in time $t$ as $\hat{\sigma}^2_{a,t} = \mathrm{thre}(\tilde{\sigma}^2_{a,t};C_{\sigma^2})$ and  define
$w_{1,t}$ and $w_{0,t}$ 
as 
\begin{align}
\label{eq:sample_ave_weight}
w_{1,t} = \frac{\sqrt{\hat{\sigma}^2_{1,t}}}{\sqrt{\hat{\sigma}^2_{1,t}} + \sqrt{\hat{\sigma}^2_{0,t}}}, \quad\mathrm{and}\quad w_{0,t} = 1 - w_{1,t}.
\end{align}
We use this strategy to apply the large deviation expansion for martingales to the estimator of the expected reward, which is the core of our theoretical analysis in Section~\ref{sec:asymp_opt}.

\begin{remark}[Remark on the sampling rule]
Unlike the sampling rule of \citet{Garivier2016}, our proposed sampling rule does not sample the next arm so that the empirical allocation ratio tracks the optimal one. This is due to the use of martingale properties under the AIPW estimator in the theoretical analysis of the upper bound. 
\end{remark}

\subsection{Recommendation rule with the AIPW estimator}
We present our recommendation rule.
In the recommendation phase in time $T$, for each $a\in\{1,0\}$, we estimate $\mu_a$ for each $a\in\{1,0\}$ and recommend a arm with the bigger estimated expected reward. 
With a truncated version of the estimated expected reward $\hat{\mu}_{a, t} = \mathrm{thre}(\tilde{\mu}_{a, t}, C_\mu)$, we define the \emph{augmented inverse probability weighting} (AIPW) estimator of $\mu_a$ as
\begin{align}
\label{eq:aipw}
\hat{\mu}^{\mathrm{AIPW}}_{a, T} =\frac{1}{T} \sum^T_{t=1}\hat{X}_{a, t},\ \mathrm{where}\ \hat{X}_{a, t} = \frac{\mathbbm{1}[A_t = a]\big(X_{a, t}- \hat{\mu}_{a, t}\big)}{w_{a,t}} + \hat{\mu}_{a, t}.
\end{align}
In the end of time $t=T$, we recommend $\hat{a}_T$ as 
\begin{align}
\label{eq:recommend}
\hat{a}_T = \begin{cases}
 &  1\quad \mathrm{if}\quad \hat{\mu}^{\mathrm{AIPW}}_{1, T} \ge \hat{\mu}^{\mathrm{AIPW}}_{0, T}, 
 \\
 & 0\quad \mathrm{otherwise}.
\end{cases}
\end{align}

The AIPW estimator $\hat{\mu}^{\mathrm{AIPW}}_{a, T}$ has the following properties: (i) its components satisfy the martingale property and hence, allow us to use the large deviation principle shown in Theorem~\ref{thm:fan_refine}; (ii) it has the smallest asymptotic variance among the possible estimators, which is needed for achieving the lower bound. 
For instance, we can use other estimators with a martingale property, such as the inverse probability weighting (IPW) estimator, but their asymptotic variance will be larger than that of the AIPW estimator. On the other hand, the sample average may have the same asymptotic variance as the AIPW estimator, but we cannot apply our large deviation bound. See also Section~\ref{sec:aipw_est}. 

We show the pseudo-code in Algorithm~\ref{alg}. We note that $C_{\mu}$ and $C_{\sigma^2}$ are introduced for technical purposes to make the estimators bounded. Therefore, we can use any large positive value. 

\begin{remark}[Remark on the recommendation rule]
We estimated the expected reward using the AIPW estimator, which has the smallest variance among the estimators. The use of the AIPW estimator is necessary from a theoretical perspective, but we conjecture that it could be replaced by a sample average estimator in the future. See Section~\ref{sec:aipw_est} for more detailed discussion.
\end{remark}
\begin{remark}[Sampling for stabilization] In the pseudo-code,  only the first two times are used for initialization. To stabilize the performance, we can increase the number of samplings in initialization. This is a similar idea to the forced-sampling \citep{Garivier2016}. In Section~\ref{sec:asymp_opt}, to show the asymptotic optimality, we use almost sure convergence of $w_{a,t}$ to $w^*_a$. As long as $w_{a,t} \xrightarrow{\mathrm{a.s}} w^*_a$, we can construct $w_{a,t}$. For instance, we can use $\tilde{w}_{a,t} = (1-r_t)w_{a,t} + r_t 1/2$ as the sampling probability instead of $w_{a,t}$, where $r_t \to 0$ as $t\to \infty$. 
\end{remark}
\begin{remark}[The role of $C_{\sigma^2}$] 
Assumption~\ref{asm:bounded_mean_variance} implies that the sampling probability is bounded by a small constant, $1/(2C_{\sigma^2}) \leq w^*_a \leq C_{\sigma^2}/2$. We assume Assumption~\ref{asm:bounded_mean_variance} so that the variance of the AIPW estimator is finite. Thus, the role of this constant looks similar to the forced sampling \citep{Garivier2016}, but different from it. We can set $C_{\sigma^2}$ sufficiently large so that it is almost negligible in implementation. 
\end{remark}

\begin{algorithm}[tb]
   \caption{NA-AIPW strategy}
   \label{alg}
\begin{algorithmic}
   \STATE {\bfseries Parameter:} Positive constants $C_{\mu}$ and $C_{\sigma^2}$.
   \STATE {\bfseries Initialization:} 
   \STATE At $t=1$, sample $A_t=1$; at $t=2$, sample $A_t=0$. For $a\in\{1,0\}$, set $w_{a, 1} = w_{a, 2} = 0.5$.
   \FOR{$t=3$ to $T$}
   \STATE Construct $w_{a, t}$ following \eqref{eq:sample_ave_weight}.
   \STATE Sample $\gamma_t$ from the uniform distribution on $[0,1]$. 
   \STATE $A_t = 1$ if $\gamma_t \leq w_{1, t}$; $A_t = 0$ if $\gamma_t > w_{1, t}$. 
   \ENDFOR
   \STATE Construct $\hat{\mu}^{\mathrm{AIPW}}_{a, T}$ for $a\in\{1,0\}$. following \eqref{eq:aipw}.
   \STATE Recommend $\hat{a}_T$ following \eqref{eq:recommend}.
\end{algorithmic}
\end{algorithm} 

\section{Asymptotic optimality of the NA-AIPW strategy}
\label{sec:asymp_opt}
In this section, we show the following upper bound of the misspecification probability of the NA-AIPW strategy, which also implies that the strategy is asymptotically optimal. We show the upper bounds for bandit models, where the rewards are sub-exponential random variables. We assume the following convergence rate assumptions on estimators of $\mu_a$ and $w^*_a$. 
\begin{assumption}
\label{asm:conv_nuisance}
For each $a\in \{1, 0\}$ and some constant $\alpha > 0$, with probability one,
\begin{align*}
& \lim_{t\to \infty} t^{\alpha}\left(\hat{\mu}_{a, t} - \mu_a\right) = 0,\ \ \  \lim_{t\to \infty} t^{\alpha}\left(w_{a, t} - w^*_a\right) = 0.
\end{align*}
\end{assumption}
Then, the upper bound is given as follows.
\begin{theorem}[Upper bound of the NA-AIPW strategy]\label{thm:optimal_2arm}
Suppose that Assumptions~\ref{asm:bounded_mean_variance} and \ref{asm:conv_nuisance} hold. 
Then, there exist positive constants $C_0, C_1 > 0$, which depends on $ C_\mu, C_{\sigma^2}$, such that for any $\nu \in \mathcal{M}$, 
$\sup_{t \in \mathbb{N}} \mathbb{E}_{\nu}\left[\exp\left(C_0  \left|\frac{\hat{X}_{1, t} -\hat{X}_{0, t} -  \Delta}{\big(\sqrt{\sigma^2_1} + \sqrt{\sigma^2_0}\big)^2}\right|\right) \;| \;\set{F}_{t-1}\right] \le C_1$. 
Then, for any $\nu \in \mathcal{M}$ such that $ 0 < \Delta \leq \left(\sqrt{\sigma^2_1} + \sqrt{\sigma^2_0}\right) \min\{ C_0/4, \sqrt{3 C_0^2/(8 C_1})\}$,
under the NA-AIPW strategy, 
\begin{align*}
    \liminf_{T \to \infty} - \frac{1}{T}\log \Pbb_\nu\left(\hat{\mu}^{\mathrm{AIPW}}_{1, T} < \hat{\mu}^{\mathrm{AIPW}}_{0, T}\right) 
\ge \frac{\Delta^2}{2 \big(\sqrt{\sigma^2_1} + \sqrt{\sigma^2_0}\big)^2} - c 
\left(\left(\frac{\Delta}{\sqrt{\sigma^2_1} + \sqrt{\sigma^2_0}}\right)^3 + \left(\frac{\Delta}{\sqrt{\sigma^2_1} + \sqrt{\sigma^2_0}}\right)^4 \right),
\end{align*}
where $c > 0$ is a constant depending on $C_0$ and $C_1$.
\end{theorem}
This theorem allows us to evaluate the exponentially small probability of misidentification up to the constant term when $\Delta\to 0$. 
From this theorem, we can obtain the following corollary, which states the optimality of our proposed strategy in the small gap.
\begin{corollary}[Asymptotic optimality of the NA-AIPW strategy]\label{cor:optimal_2arm}
Suppose that Assumptions~\ref{asm:bounded_mean_variance} and \ref{asm:conv_nuisance} hold. Then, for any $\nu \in \mathcal{M}$, under the NA-AIPW strategy, as $\Delta \to 0$,
\begin{align*}
    \liminf_{T \to \infty} - \frac{1}{\Delta^2T}\log \Pbb_\nu\left(\hat{\mu}^{\mathrm{AIPW}}_{1, T} < \hat{\mu}^{\mathrm{AIPW}}_{0, T}\right) 
\ge \frac{1}{2 \big(\sqrt{\sigma^2_1} + \sqrt{\sigma^2_0}\big)^2} - o(1).
\end{align*}
\end{corollary}

Using this corollary, when bandit models follow the Gaussian bandit models, we can show that the upper bound in Theorem~\ref{thm:optimal_2arm} matches the lower bound in Proposition~\ref{prp:lowerbound_2arms} as $\Delta \to 0$. When $\Delta \to 0$, there is no strategy under which the probability of misidentification is lower than that under the NA-AIW strategy. In other words, this result implies that the NA-AIPW estimator is asymptotically optimal in the sense that the probability of misidentification matches one of a worst-case lower bound under the small gap. 

Moreover, this result also implies that the estimation error of the optimal target allocation ratio $w^*_a$ is negligible when $\Delta\to 0$. As a result, the upper bound matches the performance of the strategies of \citet{glynn2004large} given the optimal target allocation ratio; that is, the true standard deviations are known. 
This also means that under the small gap, the estimation error of the optimal target allocation ratio does not affect the probability of misidentification.

The remaining part of this section provides proof of Theorem~\ref{thm:optimal_2arm}. 

\subsection{Cram\'er's large deviation expansions for the AIPW estimator}
Owing to the non-stationary adaptive process in BAI, it is also difficult to apply the standard large deviation bound \citep{Dembo2009large} to the simple sample average. 
For example, G\"{a}rtner-Ellis theorem \citep{Gartner1977,Ellis1984} provides a large deviation bound for dependent samples, but it requires the existence of the logarithmic moment generating function, which is not easy to be guaranteed for the samples in BAI. 

For these problems, we derive a novel Cram\'er large deviation bounds for martingales by extending the results of \citet{Grama2000} and \citet{Fan2013,fan2014generalization}. Note that their original large deviation bound is only applicable to martingales whose conditional second moment is bounded by any accuracy deterministically in advance. However, in BAI, we can only bound the conditional second moment by any accuracy in a given random path. This randomness prevents us from applying the original results of \citet{Grama2000} and  \citet{Fan2013,fan2014generalization}. 
In the following arguments, we show a large deviation bound for martingales given by the BAI strategy under the mean convergence of the unconditional second moment. Using the novel bound and the AIPW estimator, as the gap $\Delta$ goes to zero, the performance guarantee matches the lower bound.

Following \citet{Fan2013,fan2014generalization}, we derive a large deviation bound for sub-exponential random variable, which is more general than Gaussian distributions. 
\begin{assumption}
\label{asm:sub_exp}
For all $\nu\in\mathcal{M}$ and $a\in\{1,0\}$, $X_{a,t}$ is a sub-exponential random variable.
\end{assumption}

Here, we introduce key elements of our analysis. For each $t\in[T]$, we define the difference variable
\begin{align*}
\xi_{t}  = \frac{\hat{X}_{1, t} -\hat{X}_{0, t} -  (\mu_1 - \mu_0)}{\sqrt{T \left(\sigma^2_1/w_1^* + \sigma^2_0/w^*_0\right)}}
= \frac{\hat{X}_{1, t} -\hat{X}_{0, t} - \Delta}{\sqrt{T} \tilde{\sigma}},\; \mathrm{where}\; \tilde{\sigma} = \sqrt{\frac{\sigma^2_1}{w_1^*}  + \frac{\sigma^2_0}{w^*_0}}.
\end{align*}
We also define its sum $Z_t = \sum^t_{s=1}\xi_{s}$, and a sum of conditional moments $W_t = \sum^t_{s=1}\mathbb{E}_{\nu}[\xi^2_{s}| \mathcal{F}_{s-1}]$ with initialization $W_0 = 0$. By using the difference variable $\xi_{t}$, we can rewrite the estimator of the gap as $\sqrt{T}(\hat{\mu}^{\mathrm{AIPW}}_{1, T} - \hat{\mu}^{\mathrm{AIPW}}_{0, T} - \Delta) / \tilde{\sigma} = \sum^T_{t=1}\xi_{t} = Z_T$. Here, $\left\{\left(\xi_{t}, \mathcal{F}_t\right)\right\}^T_{t=1}$ is a martingale difference sequence (Appendix~\ref{appdx:martingale}), 
using the fact that $\hat{\mu}_{a,t}$ and $w_{a,t}$ are $\mathcal{F}_{t-1}$-measurable random variables.
Let us also define $V_T = \mathbb{E}_\nu [ | \sum_{t=1}^T \Ebbnu[\xi_t^2 | \Ftminone] -1 |]$
and denote the cumulative distribution function of the standard normal distribution by $\Phi(x) = ({\sqrt{2\pi}})^{-1} \int_{-\infty}^x \exp(- {t^2} / {2})\mathrm{d}t$. 
We have the following theorem on the tail probability of $Z_T= \sum^T_{t=1}\xi_{t}$. 

\begin{theorem}
\label{thm:fan_refine}
Suppose that Assumptions~\ref{asm:bounded_mean_variance}, \ref{asm:conv_nuisance}, and \ref{asm:sub_exp},  the following condition hold:\\
Condition~A: $\sup_{t\in\mathbb{N}}\mathbb{E}_{\nu}[\exp(C_0 \sqrt{T}|\xi_{t}|) \;|\; \set{F}_{t-1}]\leq C_1$ for some positive constants $C_0,C_1$.\\
Then, for some $\varepsilon > 0$ and $\alpha > 0$, there exist constants $T_0, c_1, c_2>0$ such that, for all $T\geq T_0$ and $1\leq u \leq \sqrt{T}\min\{ C_0/4, \sqrt{{3 C_0^2} / ({8 C_1})}\}$, 
\mkato{
\begin{align*}
\frac{\mathbb{P}_{\nu}\left(Z_T \leq - u\right)}{\Phi(-u)} &  \le c_1 u \exp\left(c_2\left( \frac{u^3}{\sqrt{T}} + \frac{u^4}{T} + u^2 (V_T + \varepsilon / \{T^{\alpha}(1 - \alpha)\} ) + T_0\right)  \right),
\end{align*}
}
where the constants $c_1,c_2$ depend on $C_0$ and $C_1$ but do not depend on $\{(\xi_t, \mathcal{F}_t)\}^T_{t=1}$, $u$, and the bandit model $\nu$, and $\alpha > 0$ is some constant depending on the convergence rate of $\hat{\mu}_{a,t}$ and $w_{a,t}$. 
\end{theorem}
Here, as explained by \citet{fan2014generalization}, if $T\mathbb{E}[\xi^2_t|\mathcal{F}_{t-1}]$ are all bounded from below by a positive constant, Condition~A implies the conditional Bernstein condition: for a positive constant $C$, $|\mathbb{E}[\xi^k|\mathcal{F}_{t-1}]| \leq \frac{1}{2} k!(C/\sqrt{T})^{k-2}\mathbb{E}[\xi^2_t|\mathcal{F}_{t-1}]$ for all $k\geq 2$ and all $t\in[T]$. Thus, the condition is sufficiently general and holds for our strategy under appropriate assumptions. 

Here, we provide the proof sketch of Theorem~\ref{thm:fan_refine}. The formal proof is shown in Appendix~\ref{appdx:proof_large_deviation}. 
\begin{proof}[Proof sketch of Theorem~\ref{thm:fan_refine}.] Let us define $r_t(\lambda ) = \exp(\lambda \xi_t)/\mathbb{E}[\exp(\lambda \xi_t)]$. Then, we  apply the change-of-measure in \cite{Fan2013, fan2014generalization} to transform the bound. In \citet{Fan2013,fan2014generalization}, the proof is complete up to this procedure. However in our case, unlike their settings, the second moment is also a random variable. Because of the randomness, there remains a term $\mathbb{E}[\exp(\overline{\lambda}(u)\sum^T_{t=1} \xi_t)]/(\prod^T_{t=1}\mathbb{E}[\exp(\overline{\lambda}(u) \xi_t)])$, where $\overline{\lambda}(u)$ is some positive function of $u$. Therefore, we next consider the bound of the conditional second moment of $\xi_t$ to apply $L^r$-convergence theorem (Proposition~\ref{prp:lr_conv_theorem}). With some computation, we complete the proof.
\end{proof}

From Theorem~\ref{thm:fan_refine}, for $u = \sqrt{T} { \Delta} / {\tilde{\sigma}}$ and  $\mathbb{P}_{\nu}(Z_T \leq - \sqrt{T} { \Delta} / {\tilde{\sigma}})  = \mathbb{P}_{\nu}( \sum_{t =1}^T ({\hat{X}_{1,t} - \hat{X}_{0,t} - \Delta})/({\sqrt{T} \tilde{\sigma}} )\le - \sqrt{T} { \Delta} / {\tilde{\sigma}} )
= \mathbb{P}_{\nu}(\hat{\mu}^{\mathrm{AIPW}}_{1, T} \le  \hat{\mu}^{\mathrm{AIPW}}_{0, T})$,
the probability that we fail to make the correct arm comparison is bounded as 
\begin{align*}
    \frac{\mathbb{P}_{\nu}\left(\hat{\mu}^{\mathrm{AIPW}}_{1, T} \le \hat{\mu}^{\mathrm{AIPW}}_{0, T} \right)}{\Phi(-\frac{\sqrt{T} \Delta}{\tilde{\sigma}})} \le c_1 \sqrt{T} \frac{ \Delta}{\tilde{\sigma}} \exp\left(c_2 T \left( \left(\frac{\Delta}{\tilde{\sigma}}\right)^3 +\left(\frac{\Delta}{\tilde{\sigma}}\right)^4 +\left(\frac{\Delta}{\tilde{\sigma}}\right)^2(V_T + \varepsilon  / \{T^{\alpha}(1 - \alpha)\})\right) + c_2T_0  \right).
\end{align*}

\subsection{Gaussian approximation under a small gap}
Finally, we consider an approximation of the large deviation bound. Here, $\Phi(-u)$ is bounded as $\frac{1}{\sqrt{2 \pi} (1 + u)} \exp(- \frac{u^2}{2})\le \Phi(-u) \le \frac{1}{\sqrt{\pi} (1 + u)} \exp( - \frac{u^2}{2}),\; u\ge 0$ (see  \citet[Section 2.2.,][]{Fan2013}).
\mkato{By combining this bound with Theorem~\ref{thm:fan_refine} and Proposition~\ref{prp:rate_clt} in Appendix~\ref{appdx:prelim}, which shows the rate of convergence in the CLT for $0 \leq u \leq 1$, we have the following corollary.}
\begin{corollary}
\label{thm:fan_refine2}
Suppose that  Assumptions~~\ref{asm:bounded_mean_variance}, \ref{asm:conv_nuisance}, and \ref{asm:sub_exp}, Condition~A in Theorem~\ref{thm:fan_refine}, and the following conditions hold:\\
Condition~B: ${ \Gap} / {\tilde{\sigma}} \le \min\{ C_0/4, \sqrt{3 C_0^2/(8 C_1})\}$;~~~ Condition~C: $\lim_{T \to \infty}V_T = 0$.\\
Then, there exists a constant $c>0$ such that
\begin{align*}
    \liminf_{T \to \infty} - \frac{1}{T}\log\mathbb{P}_{\nu}\left(\hat{\mu}^{\mathrm{AIPW}}_{1, T} \le \hat{\mu}^{\mathrm{AIPW}}_{0, T} \right)  \ge  \frac{\Delta^2}{2\tilde{\sigma}^2} - c 
\left(\left(\frac{\Delta}{\tilde{\sigma}}\right)^3 + \left(\frac{\Delta}{\tilde{\sigma}}\right)^4 \right).
\end{align*}
\end{corollary}
This approximation can be thought of as a Gaussian approximation because the probability
is represented by $\exp(-{\Delta^2} T / ({2 \tilde{\sigma}^2}))$. 
Condition~B is satisfied as $\Delta \to 0$. 
To use Corollary~\ref{thm:fan_refine2}, we need to show that Conditions~A and C hold.  
First, the following lemma states that Condition~A holds with the constants $C_0$ and $C_1$, which are universal to the problems in $\mathcal{M}$.
\begin{lemma}
\label{lem:condition1}
Suppose that Assumptions~\ref{asm:bounded_mean_variance} and ~\ref{asm:sub_exp} and hold. For any $\nu \in \mathcal{M}$ and each $C_0 \ge 0$, there exists a positive constant $C_1$, which depends on $ C_0, C_\mu, C_{\sigma^2}$, such that 
$\sup_{t\in\mathbb{N}} \mathbb{E}_{\nu}[\exp(C_0 \sqrt{T} |\xi_{t}|) \;| \;\set{F}_{t-1}] \le C_1$.
\end{lemma}
Next, regarding Condition~C, we introduce the following lemma for the convergence of $V_T$, which also means the mean convergence of the variance of the AIPW estimator scaled with $\sqrt{T}$. 
\begin{lemma}
\label{lem:condition2}
 Suppose that Assumptions~\ref{asm:bounded_mean_variance}, \ref{asm:conv_nuisance}, and \ref{asm:sub_exp} hold. For any $\nu \in \mathcal{M}$, $\lim_{T \to \infty}T^\alpha V_T = 0$; that is, for any $\delta > 0$, there exists $T_0$ such that for all $T>T_0$, $T^\alpha\mathbb{E}_\nu [| \sum_{t=1}^T \Ebbnu[\xi_t^2 | \Ftminone] - 1 |] \leq \delta.$
\end{lemma}

The proofs of Lemma~\ref{lem:condition1} and Lemma~\ref{lem:condition2} are shown in Appendix~\ref{appdx:lem:condition1} and \ref{appdx:lem:condition2}, respectively. 

\begin{remark}[Central limit theorem (CLT)]
Note that the CLT cannot provide the exponentially small evaluation of the probability of misidentification. It gives an approximation around $1/\sqrt{T}$ of the expected reward, but we are interested in evaluation with constant deviation from the expected reward. 
\end{remark}

\section{Experiments}
\label{sec:experiments}
We compare the proposed NA-AIPW (NA) strategy with the alpha-elimination \citep[Alpha,][]{Kaufman2016complexity} and uniform sampling strategies (Uniform). The alpha-elimination strategy is an oracle strategy, which assumes that  the variances are known. See \citet{Kaufman2016complexity} for more details of the alpha-elimination strategy. In the uniform sampling strategy, we sample arm $1$ if the time is odd and sample $0$ otherwise; finally, we recommend an arm with the highest average reward defined as $\tilde{\mu}_{a,T+1}$ in Section~\ref{sec:track_aipw}. Besides, we also investigate the performances of the following three additional strategies :\\
    \textbf{NA-DR (ND) strategy.} We replace the AIPW estimator in the NA-AIPW strategy with the DR estimator defined as $\hat{\mu}^{\mathrm{DR}}_{a, T} =\frac{1}{T} \sum^T_{t=1}\hat{X}^\dagger_{a, t}$, where $\hat{X}^\dagger_{a,t} = \frac{\mathbbm{1}[A_t = a]\big(X_{a, t}- \hat{\mu}_{a, t}\big)}{\frac{1}{t-1}\sum^{t-1}_{s=1}\mathbbm{1}[A_s = a]} + \hat{\mu}_{a, t}$.\\
    \textbf{NA-IPW (NI) strategy.} We replace the AIPW estimator in the NA-AIPW strategy with the inverse probability weighting (IPW) estimator defined as $\hat{\mu}^{\mathrm{IPW}}_{a, T} =\frac{1}{T} \sum^T_{t=1}\hat{X}^\diamond_{a, t}$, where $\hat{X}^\diamond_{a,t} = \frac{\mathbbm{1}[A_t = a]X_{a, t}}{w_{a,t}}$.\\
    \textbf{NA-SA (NS) strategy.} We replace the AIPW estimator in the NA-AIPW strategy with the simple sample average defined as $\hat{\mu}_{a, T}^{\mathrm{SA}} = \frac{1}{\sum^{T}_{t=1}\mathbbm{1}[A_t = a]}\sum^{T}_{t=1}\mathbbm{1}[A_t = a] X_{a, t}$.\\
The DR estimator replaces the allocation probability $w_{a,t}$ with its estimator. 
The IPW estimator does not use the adjustment term $\hat{\mu}_{a,t}$. 
The sample average corresponds to the IPW estimator whose allocation probability is replaced with its estimator $\frac{1}{T}\sum^{T}_{t=1}\mathbbm{1}[A_t = a]$ because $\hat{\mu}_{a, T}^{\mathrm{SA}}$ can be rewritten as $\hat{\mu}_{a, T}^{\mathrm{SA}} =  \frac{1}{T}\sum^{T}_{t=1}\frac{\mathbbm{1}[A_t = a] X_{a, t}}{\frac{1}{T}\sum^{T}_{t=1}\mathbbm{1}[A_t = a]}$. 

Besides, to stabilize the allocation probability $w_{a,t}$, instead of directly using $w_{a,t}$, we use $\tilde{w}_{a,t} = r_t \frac{1}{2} + (1-r_t) w_{a,t}$ 
for $r_t$ such that $r_t \to 0$ as $t\to\infty$. 
This prevents $w_{a,t}$ from being some extreme value. Note that $w^\gamma_{a,t} \xrightarrow{\mathrm{a.s.}} w^*_{a}$ if $w_{a,t} \xrightarrow{\mathrm{a.s.}} w^*_{a}$. 
In this experiments, we use $w^\gamma_{a,t}$ with $\gamma_t = \frac{1}{\sqrt{t}}$ and $100$ initialization times for the NA-AIPW, NA-DR, NA-IPW, and NA-SA strategies.

We consider the following eight sample scenarios: (1) $\mu_1 = 0.05$, $\mu_0 = 0.01$, $\sigma^2_1 = 1$, and $\sigma^2_0 = 0.2$; (2) $\mu_1 = 0.05$, $\mu_0 = 0.01$, $\sigma^2_1 = 1$, and $\sigma^2_0 = 0.1$; (3) $\mu_1 = 0.05$, $\mu_0 = 0.03$, $\sigma^2_1 = 1$, and $\sigma^2_0 = 0.2$; (4) $\mu_1 = 0.05$, $\mu_0 = 0.01$, $\sigma^2_1 = 1$, and $\sigma^2_0 = 0.1$; (5) $\mu_1 = 0.8$, $\mu_0 = 0.75$, $\sigma^2_1 = 5$, and $\sigma^2_0 = 3$; (6) $\mu_1 = 0.8$, $\mu_0 = 0.75$, $\sigma^2_1 = 5$, and $\sigma^2_0 = 1$; (7) $\mu_1 = 0.8$, $\mu_0 = 0.79$, $\sigma^2_1 = 5$, and $\sigma^2_0 = 3$; (8) $\mu_1 = 0.8$, $\mu_0 = 0.79$, $\sigma^2_1 = 5$, and $\sigma^2_0 = 1$. We set $T = 1,000$ and ran $10,000$ independent trials for each setting. Although we set $T = 1,000$, we save the recommended arm for each $t\in\{1,2,\dots, 10000\}$. Then, by taking the average over $10,000$ independent trials, we compute the empirical probability of identification for $t\in\{1,2,\dots, 3000\}$.  The results of Scenarios~1 and 2 are shown in Figure~\ref{fig:synthetic_results}. The other results are shown in Appendix~\ref{sec:appdx_experiments}.

We can confirm that the NA-AIPW (NA), NA-DR (ND), and NA-SA (NS) strategies perform as well as the alpha-elimination. On the other hand, the uniform sampling and NA-IPW strategies show sub-optimal results. The sub-optimal performance of the NA-IPW strategy can be considered that the IPW estimator has a larger variance than that of the AIPW estimator; in fact, the upper bound of the NA-IPW strategy does not match the lower bound, as will be shown in the following. 

As well as Lemma~\ref{lem:condition2}, we can show the following corollary. If we use the NA-IPW strategy, we cannot achieve the lower bound owing to the variance $\kappa$, which is different from that of Lemma~\ref{lem:condition2}.
\begin{corollary}
\label{lem:condition_ipw}
Under Assumptions~\ref{asm:bounded_mean_variance} and \ref{asm:sub_exp}, for any $\delta > 0$, there exists $T_0$ such that for all $T>T_0$, $\mathbb{E}[| \frac{1}{T \kappa}\sum^T_{t=1}\mathbb{E}_{\nu}[(\hat{X}_{1, t} - \hat{X}_{0, t}- \Gap)^2| \mathcal{F}_{t-1}] - 1 | ] \le \delta$,  where
$\kappa = {\mathbb{E}_{\nu}[X^2_{1,t}]} / {w^*_1} + {\mathbb{E}_{\nu}[X^2_{0,t}]} / {w^*_0}$.
\end{corollary}

We conjecture that the reason why the NA-DR strategy performs well is that the re-estimated allocation probability mitigates the fluctuation of the allocation probability $w_{a,t}$. A similar phenomenon is reported by \citet{Kato2021adr} in the context of off-policy evaluation as a paradox because even if we know the true allocation probability, replacing it with an estimator improves the performance. 

We also conjecture that the NA-SA strategy also has the same asymptotically optimal properties as our proposed NA-AIPW strategy. See the discussion in Section~\ref{sec:aipw_est}. 

\begin{figure}[t]
    \begin{tabular}{cc}
      \begin{minipage}[t]{0.45\hsize}
        \centering
        \includegraphics[width=60mm]{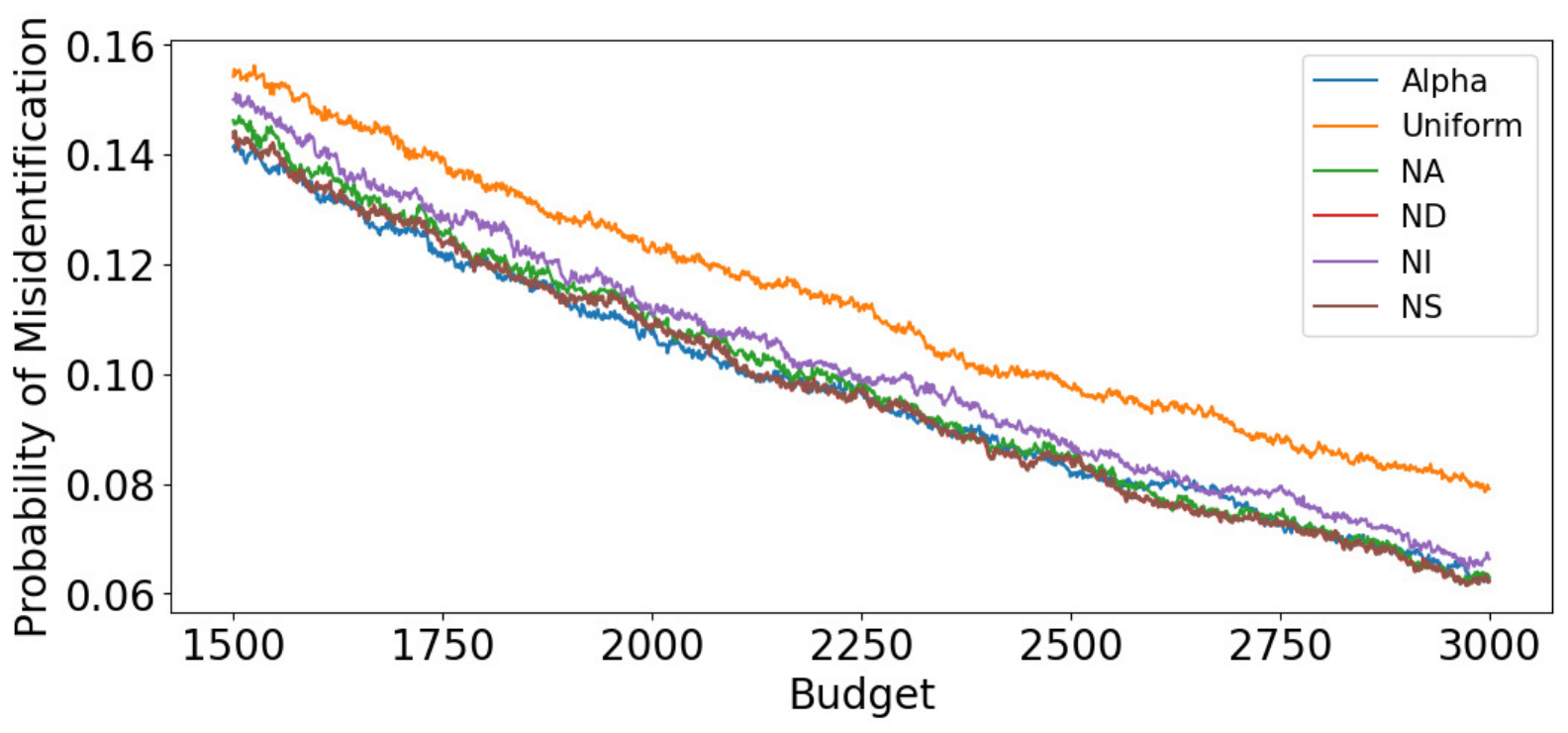}
        \caption*{Scenario~1}
      \end{minipage} &
      \begin{minipage}[t]{0.45\hsize}
        \centering
        \includegraphics[width=60mm]{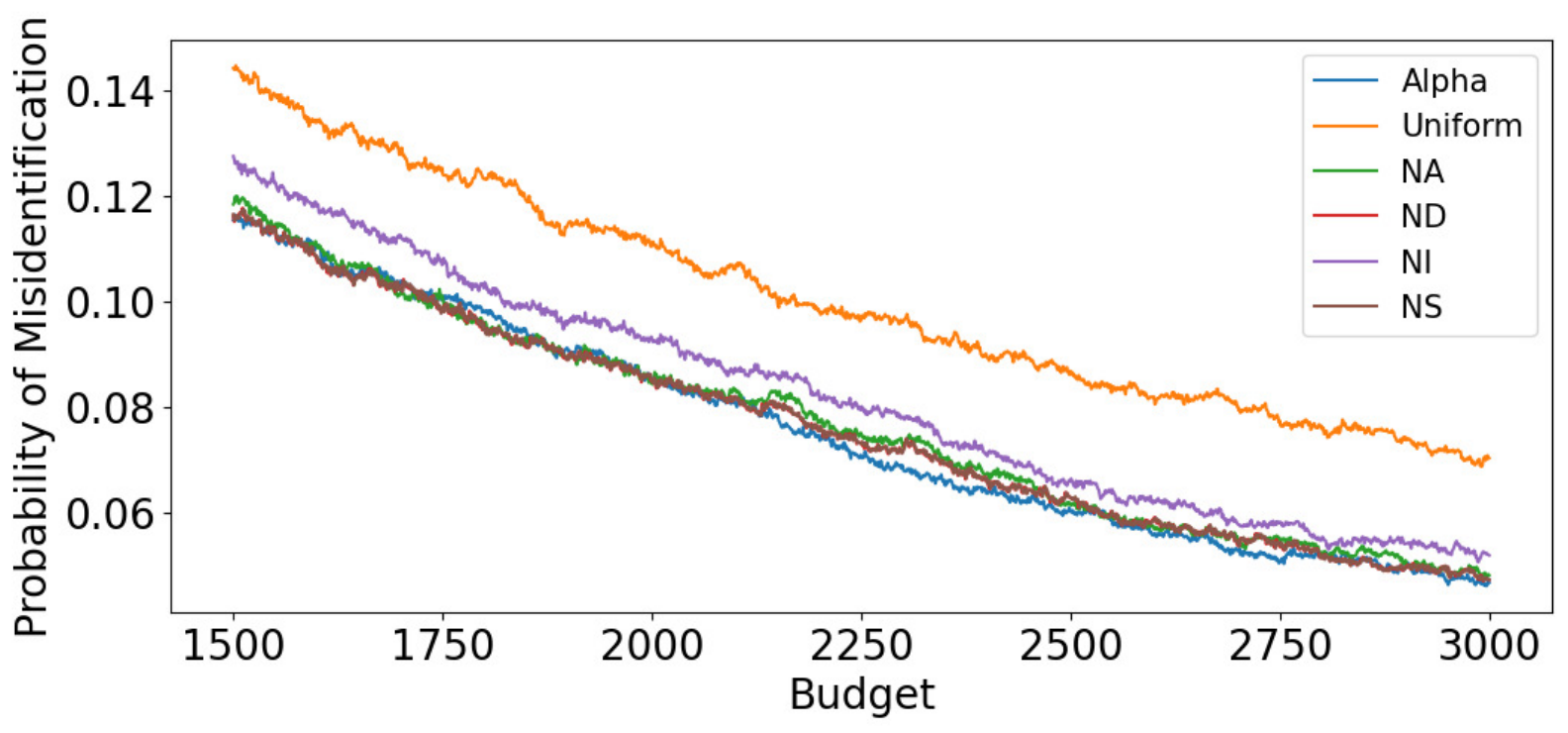}
        \caption*{Scenario~2}
      \end{minipage} 
    \end{tabular}
    \caption{Results of the two-armed Gaussian bandits. We compute the empirical probability of misidentification. Note that the result of ND is overlapped with that of the NS.}
\label{fig:synthetic_results}
\vspace{-0.5cm}
\end{figure}

\section{Discussion and related work} 
\label{sec:related}
This section presents discussions and related works. More details are shown in Appendix~\ref{sec:related2}.

\subsection{Arguments on the asymptotic optimality of BAI with a fixed budget}
\citet{Kaufman2016complexity} derives distribution-dependent lower bounds for BAI with fixed confidence and a fixed budget, based on the change-of-measure arguments as well as \citep{Lai1985}. In BAI with fixed confidence, \citet{Garivier2016} provides a strategy whose upper bound matches the lower bound. In contrast, in the fixed-budget setting, it was unclear whether there was a strategy whose upper bound matches the lower bound of \citet{Kaufman2016complexity}. We consider that this is because the estimation error of an optimal target allocation ratio is negligible in BAI with a fixed budget, unlike BAI with fixed confidence, where we can sample each arm until the strategy satisfies a condition. Besides, there are lower bounds different from \citet{Kaufman2016complexity}, such as \citet{Audibert2010}, \citet{Bubeck2011}, and \citet{Carpentier2016}. \mkato{Recently, \citet{Russo2016}, \citet{Qin2017}, \citet{Shang2020}, and \citet{Jourdan2022} propose the Bayesian BAI methods, which are optimal in the sense of the convergence rate of the posterior. Although some of their methods are shown to achieve the lower bounds of \citet{Kaufman2016complexity} for fixed-confidence BAI, their methods also do not achieve that for fixed-budget BAI. Although the posterior convergence is also optimal in the fixed-budget setting, it does not imply the optimality of the probability of misidentification. For this problem, also see \citet{Kasy2021} and \citet{Ariu2021}.} Another literature is ordinal optimization in the operation research community \citep{Dohyun2021,chen2000,glynn2004large}. Most of those studies have considered the estimation of the allocation ratio separately from the error rate under known optimal target allocation ratios. There are slaos several studies using local asymptotics to discuss the asymptotic optimality of the Neyman allocation rule in this problem \citep{Armstrong2022,adusumilli2022minimax} 

\subsection{The AIPW, IPW, and sample average Estimators}
\label{sec:aipw_est}
The key ingredient of our analysis is the AIPW estimator, which consists of a martingale difference sequence and has the minimum asymptotic variance. By focusing on the martingale difference sequence, we address the dependence among observations. Then, we derive the large deviation bound. Our large deviation bound also can be applied to the IPW estimator; however, in this case, the upper bound may not match the lower bound. This is because the asymptotic variance of the AIPW estimator is known to be smaller than that of the IPW estimator but the same as the sample average when there is no selection bias \citep{Hahn2011}. Note that the minimum variance property of the AIPW estimator is derived from the efficient influence function \citep{hahn1998role,Tsiatis2007semiparametric}.

We also conjecture that we can show the asymptotic optimality of strategies using the naive sample average estimator in the recommendation rule, although we do not show it in this study. The reason why is because \citet{Hahn2011} shows that when using the CLT, the AIPW and sample average estimators have the same asymptotic distribution in a setting similar to BAI with a fixed budget. According to their analysis and our experimental results, we also consider that these have the same large deviation property. This conjecture also implies that the NA-DR and NA-SA strategies are also asymptotically optimal as our proposed NA-AIPW strategy. However, since we cannot use the large deviation bound for martingales to investigate the asymptotic properties of the NA-DR and NA-SA strategies, the analysis becomes difficult owing to the sample dependency. The results of \citet{hirano2003} and \citet{Hahn2011} may be helpful for solving this problem. \citet{Hahn2011} shows the asymptotic distributions of the sample average by going through that of the AIPW estimator. 

\subsection{Extension to BAI in the multi-armed bandit (MAB) problems}
Unlike two-armed bandit problems and BAI with fixed confidence, the lower bounds are also still unknown for the MAB problems. One of the main reasons is that the KL-divergence is reversed. Under our small gap framework, we can consider two approaches. First, assuming the gap between the top two arms is small, we may directly extend our results. Similarly, the second direction is to assume that the gap among all arms is small; a similar approach could be used to realize proof in the multi-armed case, although we may need more assumptions.

\section{Conclusion}
\label{sec:conclusion}
We characterized the optimal probability of misidentification for BAI with a fixed budget when the gap between the expected rewards is small. 
With the help of a new large deviation expansion we developed, we showed that the performance of our proposed NA-AIPW strategy matches the lower bound under a small gap. Although there are several remaining open questions, our result provides insight into long-standing open problems in BAI. Furthermore, we can extend our results to BAI with multiple arms and contextual information, such as \citep{Kato2021Role,Russac2021,QinRusso2022,Kato2022semipara}.

\bibliographystyle{icml2022}
\bibliography{TGBAI.bbl}

\clearpage 

\appendix

\tableofcontents

\section{Preliminaries}
\label{appdx:prelim}

\begin{definition}\label{dfn:uniint}[Uniformly integrable, \citet{Hamilton1994}, p.~191]  A sequence $\{A_t\}$ is said to be uniformly integrable if for every $\epsilon > 0$ there exists a number $c>0$ such that 
\begin{align*}
\mathbb{E}_{\nu}[|A_t|\cdot I[|A_t| \geq c]] < \epsilon
\end{align*}
for all $t$.
\end{definition}
The following proposition is from \citet{Hamilton1994}, Proposition~7.7, p.~191.
\begin{proposition}[Sufficient conditions for uniform integrability]\label{prp:suff_uniint}  (a) Suppose there exist $r>1$ and $M<\infty$ such that $\mathbb{E}_{\nu}[|A_t|^r]<M$ for all $t$. Then $\{A_t\}$ is uniformly integrable. (b) Suppose there exist $r>1$ and $M < \infty$ such that $\mathbb{E}_{\nu}[|b_t|^r]<M$ for all $t$. If $A_t = \sum^\infty_{j=-\infty}h_jb_{t-j}$ with $\sum^\infty_{j=-\infty}|h_j|<\infty$, then $\{A_t\}$ is uniformly integrable.
\end{proposition}

\begin{proposition}[$L^r$ convergence theorem, p~165, \citet{loeve1977probability}]
\label{prp:lr_conv_theorem}
Let $0<r<\infty$, suppose that $\mathbb{E}_{\nu}\big[|a_n|^r\big] < \infty$ for all $n$ and that $a_n \xrightarrow{\mathrm{p}}a$ as $n\to \infty$. The following are equivalent: 
\begin{description}
\item{(i)} $a_n\to a$ in $L^r$ as $n\to\infty$;
\item{(ii)} $\mathbb{E}_{\nu}\big[|a_n|^r\big]\to \mathbb{E}_{\nu}\big[|a|^r\big] < \infty$ as $n\to\infty$; 
\item{(iii)} $\big\{|a_n|^r, n\geq 1\big\}$ is uniformly integrable.
\end{description}
\end{proposition}

\begin{proposition}[Strong law of large numbers for martingales, p 35,  \cite{hall1980martingale}]\label{prp:law_large_num_Hall}
Let $\{S_n = \sum^{n}_{i=1} X_i, \mathcal{F}_{n}, n\geq 1\}$ be a martingale and $\{U_n, n \geq 1\}$ a nondecreasing sequence of positive r.v. such that $U_n$ is $\mathcal{F}_{n-1}$-measurable. Then, 
\begin{align*}
    \lim_{n \to \infty}U^{-1}_n S_n= 0
\end{align*}
almost surely on the set $\{\lim_{n\to \infty} U_n = \infty, \; \sum_{i=1}^\infty U_i^{-1} \Ebb[|X_i| \;| \mathcal{F}_{i-1}] < \infty\}$.
\end{proposition}

\begin{proposition}[Rate of convergence in the CLT, From Theorem~3.8, p 88,  \cite{hall1980martingale}]\label{prp:rate_clt}
Let $\{S_t = \sum^{t}_{s=1} X_s, \mathcal{F}_{t}, t\geq 1\}$ be a martingale with $\mathcal{F}_{t}$ equal to the $\sigma$-field generated by $X_1,\dots, X_t$. Let 
\begin{align*}
    V^2_t = \mathbb{E}\left[\left|\sum^t_{s=1}\mathbb{E}[X^2_s| \mathcal{F}_{s-1}] - 1\right|\right]\qquad 1 \leq t \leq T.
\end{align*}
Suppose that for some $\alpha > 0$ and constants $M$, $C$ and $D$, 
\begin{align*}
    \max_{s\leq t}\mathbb{E}[\exp(|\sqrt{T} X_t|^{\alpha})] < M,
\end{align*}
and 
\begin{align*}
    \mathbb{P}\left(|V^2_t - 1 | > D/\sqrt{t}(\log t)^{2+2/\alpha}\right) \leq C t^{-1/4}(\log t)^{1+1/\alpha}.
\end{align*}
Then, for $T\geq 2$,
\begin{align}
    \sup_{-\infty < x < \infty}\big|\mathbb{P}(S_T \leq x) - \Phi(x)\big|\leq A T^{-1/4} (\log T)^{1+1/\alpha},
\end{align}
where the constant $A$ depends only on $\alpha$, $M$, $C$, and $D$.
\end{proposition}

\section{Distribution-dependent lower bound}
\label{app:lowerbound_2arms}
First, we recap the transportation lemma of \citet{Kaufman2016complexity}, which is used to derive the lower bounds. Then, we show the lower bound for Gaussian bandit models (Proposition~\ref{prp:lowerbound_2arms}).

\subsection{Recap of the transportation lemma}
The distribution-dependent lower bounds in the regret minimization problem have been popularized by \citet{Lai1985}, which uses the changes-of-measure arguments to derive them. For BAI, \citet{Kaufman2016complexity} suggests the distribution-dependent lower bounds specific to each reward distribution. 

Let $\mathrm{Alt}(\nu)\in\mathcal{M}$ be an alternative hypothesis, where $\nu'\in\mathrm{Alt}(\nu)$ has a different optimal arm from $a^*(\nu) = 1$; that is, $a^*(\nu') = 0$. Then, \citet{Kaufman2016complexity} gives the following information-theoretic lower bound.
\begin{proposition}
[From Theorem~12 in \citet{Kaufman2016complexity}]\label{prp:general_lowerbound_2arms}
Under Assumption~\ref{asm:bounded_mean_variance}, for each $\nu = (\nu_1, \nu_0) \in \mathcal{M}$, any consistent algorithm satisfies
\begin{align*}
\limsup_{T \to \infty} - \frac{1}{T}\log \mathbb{P}_\nu(\hat{a}_T \neq a^*(\nu)) \le \inf_{(\nu'_1, \nu'_0)\in\mathrm{Alt}(\nu)} \max \left\{\mathrm{KL}(\nu'_1, \nu_1), \mathrm{KL}(\nu'_0, \nu_0)\right\}.
\end{align*}

For probability distributions $p$ and $q$, when $p$ is absolutely continuous with respect to $q$, we define the Kullback-Leibler (KL) divergence 
as $\KL(p, q) = \int \log(\frac{ \mathrm{d} p(x)}{\mathrm{d} q(x)}) \mathrm{d}p(x)$. We denote the KL divergence of Bernoulli distribution by  $d(x, y) := x \log(x/y) + (1 - x)\log((1 - x)/(1 - y))$ with the convention that $d(0, 0) = d(1, 1) = 0$. 

\end{proposition}
Based on the results, we show the lower bound for Gaussian bandit models in Proposition~\ref{prp:lowerbound_2arms}, which is mainly based on the arguments in \citet{Kaufman2016complexity}. 

Let us define $N_{a,\tau}=\sum^\tau_{t=1}\mathbbm{1}[A_t = a]$ where $\tau$ is the stopping time with respect to $(\mathcal{F}_t)_{t\ge 1}$.
We follow a proof procedure similar to Theorem~12 in \cite{Kaufman2016complexity}.
Consider a perturbed bandit model $\nu'$, where $\mu_0' > \mu_1'$. We use the following lemma from \cite{Kaufman2016complexity}.
\begin{lemma}[Lemma~1 in \cite{Kaufman2016complexity}]\label{lem:data_proc_inequality}
Let $\nu$ and $\nu'$ be two bandit models with $K$ arms such that for all $a \in \{1, 0\}$, $\nu_a$ (distribution of reward of arm $a$ for model $\nu$), and $\nu_a'$ (distribution of reward of arm $a$ for the model $\nu'$) are  mutually absolutely continuous.
For an almost surely finite stopping time $\tau$ with respect to $(\mathcal{F}_t)_{t\ge 1}$, 
\begin{align}\label{eq:kauf_change_of_measure}
    \sum_{a\in\{1, 0\}} \Ebb_\nu[N_{a,\tau}] \KL(\nu_a, \nu_a') \ge \sup_{\set{E} \in \set{F}_\tau} d(\Pbb_{\nu}(\set{E}),\Pbb_{\nu'}(\set{E})),
\end{align}
where $d(x, y) := x \log(x/y) + (1 - x)\log((1 - x)/(1 - y))$ is the binary relative entropy with the convention that $d(0, 0) = d(1, 1) = 0$.
\end{lemma}
Event $\mathcal{E}$ on the right hand side of \eqref{eq:kauf_change_of_measure} is measurable (by $\mathcal{F}_\tau$) event. Note that the inequality \eqref{eq:kauf_change_of_measure} becomes tight when selecting an event that is discriminative for the two models, $\nu$ and $\nu'$, such as $\Pbb_\nu(\set{E}) \to 1$ and $\Pbb_{\nu}(\set{E}) \to 0$. Subsequently, we select $\mathcal{E} = \{\hat{a}_T = 1\}$ as a discriminative event. 

By applying Lemma~\ref{lem:data_proc_inequality}, we have
\begin{align*}
    \Ebb_{\nu'}[N_{1,T}]\KL(\nu_1', \nu_1) + \Ebb_{\nu'}[N_{0,T}]\KL(\nu'_0, \nu_0)  \ge \sup_{\set{E} \in \set{F}_T}d(\Pbb_{\nu'}(\set{E}), \Pbb_{\nu}(\set{E})).
\end{align*}

Let $\mathcal{E} = \{\hat{a}_T = 1\}$. 
Because we assume that the strategy is consistent for both models, for each $\varepsilon  \in (0, 1)$,
there exists $t_0 (\varepsilon)$ such that for all $T \ge t_0 (\varepsilon)$, 
\begin{align*}
    \Pbb_{\nu'}(\mathcal{E}) \le \varepsilon \le \Pbb_{\nu}(\mathcal{E}).
\end{align*}
Then, for all $T \ge t_0(\varepsilon)$,
\begin{align*}
    &\Ebb_{\nu'}[N_{1,T}]\KL(\nu_1', \nu_1) + \Ebb_{\nu'}[N_{0,T}]\KL(\nu'_0, \nu_0)\\
    \\
    & \ge d(\Pbb_{\nu'}(\mathcal{E}), \Pbb_{\nu}(\mathcal{E}))
    \\
    & \ge d (\varepsilon, 1 - \Pbb_\nu(\hat{a}_T \neq 1))
    \\
    & \ge \varepsilon \log \varepsilon + (1 - \varepsilon) \log \frac{1 - \varepsilon}{\Pbb_\nu(\hat{a}_T \neq 1)}.
\end{align*}
Considering $\limsup_{T \to \infty}$ and letting $\varepsilon \to 0$, we obtain
\begin{align*}
    \limsup_{T \to \infty} \frac{1}{T}\log \frac{1}{\Pbb_\nu(\hat{a}_T \neq 1)} & \le \limsup_{T \to \infty} \frac{\Ebb_{\nu'}[N_{1,T}]}{T}\KL(\nu_1', \nu_1) + \frac{\Ebb_{\nu'}[N_{0,T}]}{T}\KL(\nu'_0, \nu_0)\\
    & \le \max_{a\in\{1, 0\}} \KL(\nu'_a, \nu_a). 
\end{align*}
Then, we have
\begin{align*}
\limsup_{T \to \infty} - \frac{1}{T}\log \mathbb{P}_\nu(\hat{a}_T \neq a^*(\nu)) \le \inf_{(\nu'_1, \nu'_0)\in\mathrm{Alt}(\nu)} \max \left\{\mathrm{KL}(\nu'_1, \nu_1), \mathrm{KL}(\nu'_0, \nu_0)\right\}.
\end{align*}

\subsection{Proof of lower bound for Gaussian bandit models (Proposition~\ref{prp:lowerbound_2arms})}
\label{appdx:prp:lowerbound_2arms}
First, we consider the class of Gaussian bandit models, $\mathcal{M} = \{\nu = (\mathcal{N}(\mu_1, \sigma^2_1), \mathcal{N}(\mu_0, \sigma^2_0)): (\mu_1, \mu_0)\in\Theta^2, \mu_1\neq \mu_0\}$, where $\mathcal{N}(\mu_a,\sigma^2_a)$ denotes the Gaussian distribution of arm $a$ with the expected reward $\mu_a$ and variance $\sigma^2_a$. 
Then, for an alternative hypothesis $\mathrm{Alt}(\nu) = \{\nu' \in \mathcal{M}: \mu'_1 < \mu'_0\}$, we obtain the following corollary, where $\mu'_a = \mathbb{E}_{\nu'}[Y_{a,t}]$.

This result is based on the arguments in Section~5.1 of  \citet{Kaufman2016complexity}. Because they do not provide the formal proof, we show it in Appendix~\ref{appdx:prp:lowerbound_2arms}.

In Gaussian bandit models, 
\begin{align*}
    \KL(\nu'_a, \nu_a) = \frac{(\mu'_a - \mu_a)^2}{2\sigma^2_a}.
\end{align*}
Then,
\begin{align*}
     \limsup_{T \to \infty} \frac{1}{T}\log \frac{1}{\Pbb_\nu(\hat{a}_T \neq 1)} & 
     \le \inf_{(\mu_1', \mu_0'): \mu_0' > \mu_1'} \max_{a \in \{1, 0\}}\frac{(\mu_a' - \mu_a)^2}{2\sigma^2_a}
     \\
     & = \min_{\lambda \in \mathbb{R}} \max_{a \in \{1, 0\}}\frac{(\lambda - \mu_a)^2}{2\sigma^2_a}.
\end{align*}
When the minimum over $\lambda \in \mathbb{R}$ is attained, 
\begin{align*}
 \frac{(\lambda - \mu_1)^2}{2\sigma^2_1 } & = \frac{(\lambda - \mu_0)^2}{2\sigma^2_0}
\\
    \lambda & = \frac{\mu_1 \sqrt{\sigma^2_0} + \mu_0 \sqrt{\sigma^2_a}}{\sqrt{2} \sqrt{\sigma^2_1} \sqrt{\sigma^2_0}}.
\end{align*}
Thus, we have 
\begin{align*}
    \limsup_{T \to \infty} \frac{1}{T}\log \frac{1}{\Pbb_\nu(\hat{a}_T \neq 1)} \le \frac{(\mu_1 - \mu_0)^2}{2 (\sqrt{\sigma^2_1} + \sqrt{\sigma^2_0})^2}.
\end{align*}
This concludes the proof.

\section{\texorpdfstring{$(\xi_{t}, \mathcal{F}_t)$}{TEXT} is martingale difference sequences}
\label{appdx:martingale}
\begin{proof}
Clearly, $ \Ebb_\nu[ |\xi_{t}|] < \infty$. For each $t \in [T]$, 
\begin{align*}
    \mathbb{E}_{\nu}\left[\xi_{t}| \mathcal{F}_{t-1}\right] & = \frac{1}{\sqrt{T } \tilde{\sigma}}\mathbb{E}_{\nu}\left[\hat{X}_{1, t} -\hat{X}_{0, t} -  \Delta\big| \mathcal{F}_{t-1}\right]
    \\
    &= \frac{1}{\sqrt{T } \tilde{\sigma}}\Bigg(\frac{\Ebbnu[\mathbbm{1}[A_t = 1] | \Ftminone]\Ebbnu\left[\big(X_{1, t}- \hat{\mu}_{1, t}\big) | \Ftminone\right]}{w_{1,t}}   +\hat{\mu}_{1,t} 
    \\
    & \qquad \qquad \qquad - \frac{\Ebbnu[\mathbbm{1}[A_t = a]| \Ftminone]\Ebbnu\left[\big(X_{a, t}- \hat{\mu}_{0, t}\big) | \Ftminone\right]}{w_{a,t}}  - \hat{\mu}_{0,t}- \Delta\Bigg)
    \\
    & =  0.
\end{align*}
\end{proof}

\section{Proof of Lemma~\ref{lem:condition1}}
\label{appdx:lem:condition1}
\begin{proof} For the simplicity, let us denote $\Ebb_\nu$ by $\Ebb$.
Recall that $\hat{X}_{a, t}$ is constructed as 
\begin{align*}
    \hat{X}_{a, t} = \frac{\mathbbm{1}[A_t = a]\big(X_{a, t}- \hat{\mu}_{a, t}\big)}{w_{a,t}} + \hat{\mu}_{a, t},
\end{align*}
where 
\begin{align*}
&\hat{\mu}_{a, t} = \begin{cases}
C_\mu&\mathrm{if}\  C_\mu < \tilde{\mu}_{a, t}\\
\tilde{\mu}_{a, t}&\mathrm{if}\ - C_\mu \leq \tilde{\mu}_{a, t} \leq C_\mu\\
- C_\mu&\mathrm{if}\ \tilde{\mu}_{a, t} < - C_\mu.
\end{cases}
\end{align*}
For each $t = 1, \ldots, T$, we have
\begin{align*}
    & \Ebb\left[\exp\left(C_0 \sqrt{T} |\xi_{t}|\right) \middle| \set{F}_{t-1}\right] 
    \\
    & = \Ebb\left[\exp \left(C_0 \left| \frac{\hat{X}_{1,t} - \hat{X}_{0,t} - \Delta}{\tilde{\sigma}}\right| \right) \middle| \set{F}_{t-1}\right] 
    \\
    & \le \Ebb\left[\exp \left( \frac{C_0}{\tilde{\sigma}} \left|\hat{X}_{1,t} - \hat{X}_{0,t} \right| + \frac{C_0 \Delta}{\tilde{\sigma}}  \right) \middle| \set{F}_{t-1}\right] 
    \\
    & \le \Ebb\left[\exp \left( \frac{C_0}{\tilde{\sigma}} \left|\hat{X}_{1,t} - \hat{X}_{0,t} \right| + \frac{2 C_0 C_\mu}{\tilde{\sigma}}  \right) \middle| \set{F}_{t-1}\right]
    \\
    & \stackrel{(a)}{=} \tilde{C}_1 \Ebb \left[\exp \left(\frac{C_0}{\tilde{\sigma}} |\hat{X}_{1, t} - \hat{X}_{0,t} | \right) \middle| \set{F}_{t-1}, A_t = 1\right] \Pbb (A_t = 1 \;|\; \set{F}_{t-1}) \\
    & \qquad + \tilde{C}_1 
\Ebb\left[\exp \left(\frac{C_0}{\tilde{\sigma}} |\hat{X}_{1, t} - \hat{X}_{0,t} | \right) \middle| \set{F}_{t-1}, A_t = 0\right] \Pbb (A_t = 0\; |\; \set{F}_{t-1}) 
\\
& = \tilde{C}_1 \Ebb \left[\exp \left(\frac{C_0}{\tilde{\sigma}} \left| \frac{\big(X_{1, t}- \hat{\mu}_{1, t}\big)}{w_{1,t}} + \hat{\mu}_{1, t} - \hat{\mu}_{0, t}\right| \right) \middle| \set{F}_{t-1}, A_t = 1\right] \Pbb (A_t = 1 \; | \; \set{F}_{t-1})
\\
& \qquad + \tilde{C}_1\Ebb \left[\exp \left(\frac{C_0}{\tilde{\sigma}} \left|- \frac{\big(X_{a, t}- \hat{\mu}_{a, t}\big)}{w_{0,t}} + \hat{\mu}_{1, t} - \hat{\mu}_{0, t}\right| \right) \middle| \set{F}_{t-1}, A_t = 0\right] \Pbb (A_t = 0\; | \; \set{F}_{t-1}),
\end{align*}
where for $(a)$, we denote $\tilde{C}_1 =  \exp\left(2C_0 C_\mu/\tilde{\sigma}\right)$. Since $ X_{a,t}$ is a sub-exponential random variable (Assumption~\ref{asm:sub_exp}), there exists some universal constant $C>0$ such that for all $ \nu\in \mathcal{M}$, for all $\lambda\ge 0$ such that $0 \le \lambda \le 1/C$, $ \Ebb[\exp(\lambda (X_{a,t} - \mu_a))] \le \exp(C^2 \lambda^2)$ (\cite{vershynin2018high}, Proposition 2.7.1).   
Note that from the assumptions that $|{\mu}_{a,t}| \le C_{ \mu}$, $ \max \{\sigma^2_0, 1/\sigma^2_0\} \le C_{\sigma^2}$, and $ | w_{a,t}| \ge C_w$ for all $ t \in \{1,\ldots, T\}$), where $C_w > 0$ is a constant that depends on $C_{\sigma^2}$. Therefore, there exists a positive constant $  C_1 ( C_0, C_\mu, C_{\sigma^2})$ such that
\begin{align*}
    \Ebb \left[\exp(C_0 \sqrt{T} |\xi_{t}|) \middle| \set{F}_{t-1}\right]  \le C_1 ( C_0, C_\mu, C_{\sigma^2}).
\end{align*}
This concludes the proof. 

\end{proof}

\section{Proof of Lemma~\ref{lem:condition2}}
\label{appdx:lem:condition2}
We first show the following lemma.
\begin{lemma}
\label{lem:as_conv}
Suppose that Assumptions~\ref{asm:bounded_mean_variance} and \ref{asm:conv_nuisance} hold. Then, with probability one,
\begin{align*}
    \lim_{t\to\infty}t^\alpha\left|\mathbb{E}\left[\left(\hat{X}_{1, t} - \hat{X}_{a, t}- \Gap\right)^2\Big| \mathcal{F}_{t-1}\right] -  \left(\frac{\sigma^2_1}{w_1^*} +\frac{\sigma^2_0}{w_a^*} \right)\right| = 0.
\end{align*}
\end{lemma}
\begin{proof}
\begin{align*}
&\mathbb{E}_{\nu}\left[\left(\hat{X}_{1, t} - \hat{X}_{0, t}- \Delta\right)^2\Big| \mathcal{F}_{t-1}\right]
\\
&= \mathbb{E}_{\nu}\left[\left(\frac{\mathbbm{1}[A_{t} = 1]\big(X_{1, t}- \hat{\mu}_{1, t}\big)}{w_{1,t}} - \frac{\mathbbm{1}[A_{t} = 0]\big(X_{0, t}- \hat{\mu}_{0, t}\big)}{w_{0,t}} + \hat{\mu}_{1, t} - \hat{\mu}_{0, t} - \Delta\right)^2 \Big| \mathcal{F}_{t-1}\right]
\\
&= \mathbb{E}_{\nu}\Bigg[\left(\frac{\mathbbm{1}[A_{t} = 1]\big(X_{1, t}- \hat{\mu}_{1, t}\big)}{w_{1,t}} - \frac{\mathbbm{1}[A_{t} = 0]\big(X_{0, t}- \hat{\mu}_{0, t}\big)}{w_{0,t}}\right)^2
\\
&\ \ \ \ \ \ + 2\left(\frac{\mathbbm{1}[A_{t} = 1]\big(X_{1, t}- \hat{\mu}_{1, t}\big)}{w_{1,t}} - \frac{\mathbbm{1}[A_{t} = a]\big(X_{0, t}- \hat{\mu}_{0, t}\big)}{w_{0,t}}\right)\left( \hat{\mu}_{1, t} - \hat{\mu}_{0, t} - \Delta\right)
\\
&\ \ \ \ \ \ + \left( \hat{\mu}_{1, t} - \hat{\mu}_{0, t} - \Delta\right)^2 |  \mathcal{F}_{t-1}\Bigg]
\\
&= \mathbb{E}_{\nu}\Bigg[\frac{\mathbbm{1}[A_{t} = 1]\big(X_{1, t}- \hat{\mu}_{1, t}\big)^2}{w^2_{1,t}} + \frac{\mathbbm{1}[A_{t} = 0]\big(X_{0, t}- \hat{\mu}_{0, t}\big)^2}{w^2_{0,t}}
\\
&\ \ \ \ \ \ + 2\left(\frac{\mathbbm{1}[A_{t} = 1]\big(X_{1, t}- \hat{\mu}_{1, t}\big)}{w_{1,t}} - \frac{\mathbbm{1}[A_{t} = 0]\big(X_{0, t}- \hat{\mu}_{0, t}\big)}{w_{0,t}}\right)\left( \hat{\mu}_{1, t} - \hat{\mu}_{0, t} - \Delta\right)
\\
&\ \ \ \ \ \ + \left( \hat{\mu}_{1, t} - \hat{\mu}_{0, t} - \Delta\right)^2 |  \mathcal{F}_{t-1}\Bigg]
\\
&= \mathbb{E}_{\nu}\left[\frac{\big(X_{1, t}- \hat{\mu}_{1, t}\big)^2}{w_{1,t}}|  \mathcal{F}_{t-1}\right] + \mathbb{E}_{\nu}\left[\frac{\big(X_{0, t}- \hat{\mu}_{0, t}\big)^2}{w_{0,t}}|  \mathcal{F}_{t-1}\right] -  \left(\hat{\mu}_{1, t} + \hat{\mu}_{0, t} - \Delta\right)^2. \label{eq:check4}
\end{align*}
Here, we used 
\begin{align*}
    &\mathbb{E}_{\nu}\Bigg[\frac{\mathbbm{1}[A_{t} = a]\big(X_{a, t}- \hat{\mu}_{a, t}\big)^2}{w^2_{a,t}}|  \mathcal{F}_{t-1}\Bigg] = \mathbb{E}_{\nu}\Bigg[\frac{w_{a,t}\big(X_{a, t}- \hat{\mu}_{a, t}\big)^2}{w^2_{a,t}}|  \mathcal{F}_{t-1}\Bigg] = \mathbb{E}_{\nu}\Bigg[\frac{\big(X_{a, t}- \hat{\mu}_{a, t}\big)^2}{w_{a,t}}|  \mathcal{F}_{t-1}\Bigg]
    \end{align*}
    and
    \begin{align*}
    &\mathbb{E}_{\nu}\Bigg[\frac{\mathbbm{1}[A_{t} = 1]\big(X_{1, t}- \hat{\mu}_{1, t}\big)}{w_{1,t}}\left( \hat{\mu}_{1, t} - \hat{\mu}_{0, t} - \Delta\right) |  \mathcal{F}_{t-1}\Bigg] 
    \\
    & = \left( \hat{\mu}_{1, t} - \hat{\mu}_{0, t} - \Delta\right)\mathbb{E}_{\nu}\Bigg[\frac{w_{1,t}\big(X_{1, t}- \hat{\mu}_{1, t}\big)}{w_{1,t}} |  \mathcal{F}_{t-1}\Bigg].
    \end{align*}
We also have 
\begin{align*}
&\mathbb{E}_{\nu}\left[\frac{\big(X_{a, t}- \hat{\mu}_{a, t}\big)^2}{w_{a,t}}|  \mathcal{F}_{t-1}\right]=\frac{\mathbb{E}_{\nu}[X^2_{a, t}] - 2\mu_a\hat{\mu}_{a, t}+ \hat{\mu}^2_{a, t}}{w_{a,t}}=\frac{\mathbb{E}_{\nu}[X^2_{a, t}] - \mu^2_a + (\mu_a - \hat{\mu}_{a, t})^2}{w_{a,t}}.
\end{align*}
Then,
\begin{align*}
&\mathbb{E}_{\nu}\left[\frac{\big(X_{1, t}- \hat{\mu}_{1, t}\big)^2}{w_{1,t}}|  \mathcal{F}_{t-1}\right] + \mathbb{E}_{\nu}\left[\frac{\big(X_{0, t}- \hat{\mu}_{0, t}\big)^2}{w_{0,t}}|  \mathcal{F}_{t-1}\right] -  \left(\hat{\mu}_{1, t} + \hat{\mu}_{0, t} - \Delta\right)^2
\\
&=\left(\frac{\mathbb{E}_{\nu}[X^2_{1, t}] - \mu^2_1 + (\mu_1 - \hat{\mu}_{1, t})^2}{w_{1,t}}\right) + \left(\frac{\mathbb{E}_{\nu}[X^2_{0, t}] - \mu^2_0 + (\mu_0 - \hat{\mu}_{0, t})^2}{w_{0,t}}\right) -  \left(\hat{\mu}_{1, t} + \hat{\mu}_{0, t} - \Delta\right)^2.
\end{align*}
Because $t^\alpha\left(\hat{\mu}_{a, t} - \mu_a\right)\xrightarrow{\mathrm{a.s.}} 0$ and $t^\alpha\left(w_{a, t} - w^*_a\right) \xrightarrow{\mathrm{a.s.}} 0$ from Assumption~\ref{asm:conv_nuisance}, with probability $1$,
\begin{align*}
&\lim_{t\to\infty}t^\alpha\left|\left(\frac{\mathbb{E}_{\nu}[X^2_{1, t}] - \mu^2_1 + (\mu_1 - \hat{\mu}_{1, t})^2}{w_{1,t}}\right) + \left(\frac{\mathbb{E}_{\nu}[X^2_{0, t}] - \mu^2_0 + (\mu_0 - \hat{\mu}_{0, t})^2}{w_{0,t}}\right) - \left(\hat{\mu}_{1, t} + \hat{\mu}_{0, t} - \Delta\right)^2 - \frac{\sigma^2_1}{w^*_{1}} - \frac{\sigma^2_0}{w^*_{0}}\right|
\\
&\leq \lim_{t\to\infty}t^\alpha\left|\frac{\mathbb{E}_{\nu}[X^2_{1, t}] - \mu^2_1}{w_{1,t}} - \frac{\sigma^2_1}{w^*_{1}}\right| +  \lim_{t\to\infty}t^\alpha\left| \frac{\mathbb{E}_{\nu}[X^2_{0, t}] - \mu^2_0}{w_{0,t}} - \frac{\sigma^2_0}{w^*_{0}}\right|\\
&\ \ \ + \lim_{t\to\infty}t^\alpha\frac{ (\mu_1 - \hat{\mu}_{1, t})^2}{w_{1,t}} + \lim_{t\to\infty}t^\alpha\frac{ (\mu_0 - \hat{\mu}_{0, t})^2}{w_{0,t}} + \lim_{t\to\infty}t^\alpha\left(\hat{\mu}_{1, t} + \hat{\mu}_{0, t} - \Delta\right)^2\\
&= 0.
\end{align*}
Note that $\mathbb{E}_{\nu}[X^2_{a, t}] - \mu^2_a = \sigma^2_a$. 
\end{proof}
By using Lemma~\ref{lem:as_conv}, we prove Lemma~\ref{lem:condition2}. 
\begin{proof}
Lemma~\ref{lem:as_conv} implies that $\frac{1}{T}\sum^T_{t=1}\mathbb{E}_{\nu}\left[\left(\hat{X}_{1, t} - \hat{X}_{0, t}- \Delta\right)^2\Big| \mathcal{F}_{t-1}\right] -   \left(\frac{\sigma^2_1}{w^*_{1}} + \frac{\sigma^2_0}{w^*_{0}}\right)\xrightarrow{\mathrm{a.s.}} 0$ and
\begin{align*}
\frac{1}{T \left(\frac{\sigma^2_1}{w^*_{1}} + \frac{\sigma^2_0}{w^*_{0}}\right)}\sum^T_{t=1}\mathbb{E}_{\nu}\left[\left(\hat{X}_{1, t} - \hat{X}_{0, t}- \Delta\right)^2\Big| \mathcal{F}_{t-1}\right] - 1  \xrightarrow{\mathrm{a.s.}} 0,
\end{align*}
Because $X_{a,t}$ is a sub-exponential random variable, and the other variables in $\hat{X}_{1,t}$ and $\hat{X}_{0,t}$ are bounded, 
\begin{align*}
    \frac{1}{T \left(\frac{\sigma^2_1}{w^*_{1}} + \frac{\sigma^2_0}{w^*_{0}}\right)}\sum^T_{t=1}\mathbb{E}_{\nu}\left[\left(\hat{X}_{1, t} - \hat{X}_{0, t}- \Delta\right)^2\Big| \mathcal{F}_{t-1}\right] - 1 
\end{align*}
is uniformly integrable from Proposition~\ref{prp:suff_uniint}. Then, from Proposition~\ref{prp:lr_conv_theorem}, for any $\delta$, there exists $T_0$ such that for all $T>T_0$
\begin{align*}
\mathbb{E}_{\nu}\left[\left| \frac{1}{T \left(\frac{\sigma^2_1}{w^*_{1}} + \frac{\sigma^2_0}{w^*_{0}}\right)}\sum^T_{t=1}\mathbb{E}_{\nu}\left[\left(\hat{X}_{1, t} - \hat{X}_{0, t}- \Delta\right)^2\Big| \mathcal{F}_{t-1}\right] - 1 \right|\right] \le \delta.
\end{align*}
This concludes the proof.
\end{proof}

\section{Proof of Theorem~\ref{thm:fan_refine}: large deviation bound for martingales}
\label{appdx:proof_large_deviation}

For brevity, let us denote $\Pbb_\nu$and  $\Ebb_\nu$ by $\Pbb$ and $\Ebb$, respectively. For all $t = 1, \ldots, T$, let us define
\begin{align*}
    r_t(\lambda ) = \frac{\exp\left(\lambda \xi_t\right)}{\mathbb{E}\left[\exp\left(\lambda \xi_t\right)\right]}
\end{align*}
and
\begin{align*}
    \eta_t(\lambda) = \xi_t - b_t(\lambda),
\end{align*}
where 
\begin{align*}
   b_t(\lambda) = \mathbb{E}[r_t(\lambda)\xi_t].
\end{align*}
Then, we obtain the following decomposition:
\begin{align*}
    Z_T = U_T(\lambda) + B_T(\lambda),
\end{align*}
where
\begin{align*}
    U_T(\lambda) = \sum^T_{t=1} \eta_t(\lambda)
\end{align*}
and
\begin{align*}
    B_T(\lambda) = \sum^T_{t=1}b_t(\lambda).
\end{align*}
Let $\Psi_T(\lambda) = \sum^T_{t=1}\log\mathbb{E}\left[\exp\left(\lambda \xi_t\right)\right]$.

Before showing the proof of Theorem~\ref{thm:fan_refine}, we show the following lemmas. In particular, Lemma~\ref{lem:34} in Appendix~\ref{appdx:proof_large_deviation} is our novel result to bound $\mathbb{E}[\exp(\overline{\lambda}(u)\sum^T_{t=1} \xi_t)]/(\prod^T_{t=1}\mathbb{E}[\exp(\overline{\lambda}(u) \xi_t)])$.  Lemmas~\ref{lem:31}--\ref{lem:33} are modifications of the existing results of \citet{Fan2013,fan2014generalization}.

\begin{lemma}
\label{lem:31}
Under Condition~A, 
\begin{align*}
    \mathbb{E}\left[|\xi_t|^k \;\middle| \; \set{F}_{t-1}\right] \leq k!\left(C_0 T^{1/2}\right)^{-k}C_1,\qquad \text{for all }\quad k \geq 2.
\end{align*}
\end{lemma}
\begin{proof}
Applying the elementary inequality $ x^k/k! \leq \exp(x),  \forall x \ge 0$,
to $x=C_0|\sqrt{T}\xi_t|$, for $k\geq 2$, 
\begin{align*}
   |\xi_t|^k \leq k!(C_0  T^{1/2})^{-k}\exp(C_0|\sqrt{T}\xi_t|).
\end{align*}
Taking expectations on both sides, with Condition~A, we obtain the
desired inequality. Recall that Condition~A is
\[\sup_{1\leq t \leq T}\mathbb{E}_{\nu}\left[\exp\left(C_0 \sqrt{T}\left|\xi_t\right|\right) \;\middle|\; \set{F}_{t-1}\right]\leq C_1\]
for some positive constants $C_0$ and $C_1$.
\end{proof}

\begin{lemma}
\label{lem:32}
Under Condition~A, there exists some constant $C>0$ such that for all $0\leq \lambda \leq \frac{1}{4}C_0 \sqrt{T}$,
\begin{align*}
     \left|B_T(\lambda) - \lambda\right| \leq C\left(\lambda V_T + \lambda^2/\sqrt{T} \right).
\end{align*}
\end{lemma}
\begin{proof}
By definition, for $t = 1,\dots, T$,
\begin{align*}
    b_t(\lambda) &= \frac{\mathbb{E}\left[\xi_t\exp\left(\lambda \xi_t\right)\right]}{\mathbb{E}\left[\exp\left(\lambda \xi_t\right)\right]}.
\end{align*}
Jensen's inequality and $\mathbb{E}[\xi_t] = \mathbb{E}[\mathbb{E}[\xi_t|\mathcal{F}_{t-1}]] = 0$ implies that $\mathbb{E}[\exp(\lambda \xi_t)]\geq 1$ and
\begin{align*}
    &\mathbb{E}\left[\xi_t\exp\left(\lambda \xi_t\right)\right] = \mathbb{E}\left[\xi_t\left(\exp\left(\lambda \xi_t\right)- 1\right)\right] \geq 0,\qquad\mathrm{for}\ \lambda \geq 0.
\end{align*}
We find that 
\begin{align*}
B_T(\lambda) &\leq \sum^T_{t=1}\mathbb{E}[\xi_t\exp(\lambda \xi_t)]\\
&= \lambda \mathbb{E}[W_T] + \sum^T_{t=1}\sum^\infty_{k=2} \mathbb{E}\left[\frac{\xi_t (\lambda \xi_t)^k}{k!}\right],
\end{align*}
by the Taylor series expansion for $\exp(x)$. Recall that $W_T = \sum^T_{t=1}\mathbb{E}_{\nu}\left[\xi^2_{t}| \mathcal{F}_{t-1}\right]$ is the sum of the conditional second moment. Here, using Lemma~\ref{lem:31} and $\mathbb{E}\left[\xi^{k+1}_t\right] = \mathbb{E}\left[\mathbb{E}\left[\xi^{k+1}_t | \set{F}_{t-1}\right]\right]$, for some constant $C_2$,
\begin{align}
    \sum^T_{t=1}\sum^\infty_{k=2}\left| \mathbb{E}\left[\frac{\xi_t (\lambda \xi_t)^k}{k!}\right]\right| &\leq \sum^T_{t=1}\sum^\infty_{k=2}\left| \mathbb{E}\left[\xi^{k+1}_t\right]\right|\frac{\lambda^k}{k!}\nonumber\\
    \label{eq:bound_lem32}
    &\leq \sum^T_{t=1}\sum^\infty_{k=2}(k+1)!\left(C_0 T^{1/2}\right)^{-(k+1)}C_1\frac{\lambda^k}{k!}\nonumber\\
    &\leq C_2\lambda^2/\sqrt{T}.
\end{align}
Therefore, 
\begin{align*}
B_T(\lambda) \leq \lambda + \lambda V_T + C_2\lambda^2/\sqrt{T}.
\end{align*}
Next, we show the lower bound of $B_T(\lambda)$. First, by using Lemma~\ref{lem:31}, using some constant $C_3>0$, for all $0\leq \lambda \leq \frac{1}{4} C_0 \sqrt{T}$,
\begin{align*}
    \mathbb{E}\left[\exp(\lambda\xi_t)\right] &\leq 1 + \sum^\infty_{k=2}\left|\mathbb{E}\left[\frac{(\lambda\xi_t)^k}{k!}\right]\right|\\
    &\leq 1 + C_1\sum^\infty_{k=2}\lambda^k(C_0 \sqrt{T})^{-k}\\
    &\leq 1 + C_3\lambda^2 T^{-1}.
\end{align*}
This inequality together with \eqref{eq:bound_lem32} implies the lower bound of $B_T(\lambda)$: for some positive constant $C_4$,
\begin{align*}
    B_T(\lambda) &=\sum^T_{t=1}\frac{\mathbb{E}\left[\xi_t\exp\left(\lambda \xi_t\right)\right]}{\mathbb{E}\left[\exp\left(\lambda \xi_t\right)\right]}\\
    &\geq \left(\sum^T_{t=1}\mathbb{E}[\xi_t \exp(\lambda \xi_t)]\right)\big(1 + C_3\lambda^2 T^{-1}\big)^{-1}\\
    &= \left(\lambda W_T + \sum^T_{t=1}\sum^\infty_{k=2} \mathbb{E}\left[\frac{\xi_t (\lambda \xi_t)^k}{k!}\right]\right)\big(1 + C_3\lambda^2 T^{-1}\big)^{-1}\\
    &\geq \left(\lambda W_T - \sum^T_{t=1}\sum^\infty_{k=2}\left| \mathbb{E}\left[\frac{\xi_t (\lambda \xi_t)^k}{k!}\right]\right|\right)\big(1 + C_3\lambda^2 T^{-1}\big)^{-1}\\
    &\geq \big(\lambda - \lambda V_T - C_2\lambda^2 /\sqrt{T}\big)\big(1 + C_3\lambda^2 T^{-1}\big)^{-1}\\
    &\geq \lambda - \lambda V_T - C_4 \lambda^2 /\sqrt{T}.
\end{align*}
This concludes the proof.
\end{proof}

\begin{lemma}
\label{lem:33}
Assume Condition~A. There exists some constant $C>0$ such that for all $0\leq \lambda \leq \frac{1}{4}C_0\sqrt{T}$,
\begin{align*}
     \left|\Psi_T(\lambda) - \frac{\lambda^2}{2}\right| \leq C\left(\lambda^3/\sqrt{T} + \lambda^2V_T\right).
\end{align*}
\end{lemma}
\begin{proof}
First, we have $\mathbb{E}\left[\exp(\lambda \xi_t)\right] \geq 1$ from Jensen's inequality.
Using Taylor expansion of $\log (1+\varphi)$, $\varphi\ge 0$, there exists $0 \leq \varphi^\dagger_t\leq \mathbb{E}\left[\exp(\lambda \xi_t)\right] - 1 \; (\text{for } t=1,\ldots, T)$ such that
\begin{align*}
    \Psi_T(\lambda) & = \log\prod^T_{t=1}\mathbb{E}\left[\exp\left(\lambda \xi_t\right)\right]\\
    &= \sum^T_{t=1}\left(\left(\mathbb{E}\left[\exp(\lambda \xi_t)\right] - 1\right) -   \frac{1}{2\left(1+\varphi^\dagger_t\right)^2}\left(\mathbb{E}\left[\exp(\lambda \xi_t)\right] - 1\right)^2\right).
\end{align*}
Because $(\xi_t)$ is a martingale difference sequence, $\mathbb{E}[\xi_t] = \mathbb{E}[\mathbb{E}[\xi_t|\mathcal{F}_{t-1}]]= 0$. Therefore,
\begin{align*}
     &\Psi_T(\lambda) - \frac{\lambda^2}{2}\mathbb{E}[W_T]
     \\
     &= \sum^T_{t=1}\left(\left(\mathbb{E}\left[\exp(\lambda \xi_t)\right] - 1\right) -   \frac{1}{2\left(1+\varphi^\dagger_t\right)^2}\left(\mathbb{E}\left[\exp(\lambda \xi_t)\right] - 1\right)^2\right) - \sum^T_{t=1}\left( \lambda \mathbb{E}[\xi_t] + \frac{\lambda^2}{2}\mathbb{E}[\xi^2_t]\right) 
\end{align*}
Then, by using $\mathbb{E}\left[\exp(\lambda \xi_t)\right] \geq 1$, we have 
\begin{align*}
    \left|\Psi_T(\lambda) - \frac{\lambda^2}{2}\mathbb{E}[W_T]\right| &\leq \sum^T_{t=1}\left|\mathbb{E}\left[\exp(\lambda \xi_t)\right] - 1 - \lambda \mathbb{E}[\xi_t] - \frac{\lambda^2}{2}\mathbb{E}[\xi^2_t]\right| + \frac{1}{2} \sum^T_{t=1}\left(\mathbb{E}\left[\exp(\lambda \xi_t)\right] - 1\right)^2\\
    &\leq \sum^T_{t=1}\sum^{+\infty}_{k=3}\frac{\lambda^k}{k!}\left|\mathbb{E}\left[\xi^k_t\right]\right| + \frac{1}{2} \sum^T_{t=1}\left(\sum^{+\infty}_{k=1}\frac{\lambda^k}{k!}\left|\mathbb{E}\left[\xi^k_t\right]\right|\right)^2.
\end{align*}
From Lemma~\ref{lem:31}, for a constant $C_3$,
\begin{align*}
    \left|\Psi_T(\lambda) - \frac{\lambda^2}{2}\mathbb{E}[W_T]\right| \leq C_3\lambda^3/\sqrt{T}
\end{align*}
In conclusion, we have
\begin{align*}
    \left|\Psi_T(\lambda) - \frac{\lambda^2}{2}\right| \leq C_3\lambda^3/\sqrt{T} + \frac{\lambda^2}{2}\left(\mathbb{E}[W_T-1]\right)\leq C_3\lambda^3 /\sqrt{T} + \frac{\lambda^2}{2}\mathbb{E}[|W_T-1|].
\end{align*}
Recall that $V_T = \mathbb{E}[|W_T-1|]$. Then, 
\begin{align*}
    \left|\Psi_T(\lambda) - \frac{\lambda^2}{2}\right| \leq C\left(\lambda^3/\sqrt{T} + \lambda^2V_T\right).
\end{align*}
\end{proof}
\mkato{
\begin{lemma}
\label{lem:34}
Assume Condition~A. For $\varepsilon > 0$ there exists $T_0 > 0$ and some constants $\tilde{C}_2,\tilde{C}_3,\tilde{C}_4 >0$ such that for all $T \geq T_0$ and $0\leq \lambda \leq \frac{1}{4}C_0\sqrt{T}$,
\begin{align}
\frac{\mathbb{E}\left[\exp\left(\overline{\lambda}\sum^T_{t=1} \xi_t\right)\right]}{\prod^T_{t=1}\mathbb{E}\left[\exp\left(\overline{\lambda} \xi_t\right)\right]} 
&\leq \exp\left(\tilde{C}_2\overline{\lambda}^4/T + \tilde{C}_3\overline{\lambda}^3/\sqrt{T} + \tilde{C}_4T_0 + \varepsilon \overline{\lambda}^2  / T^\alpha \right).
\end{align}
\end{lemma}
}
\begin{proof}
Here, we have
\begin{align*}
    \mathbb{E}\left[\exp\left(\overline{\lambda}\sum^T_{t=1} \xi_t\right)\right] = \mathbb{E}\left[\prod^T_{t=1}\mathbb{E}\left[\exp\left(\overline{\lambda} \xi_t\right)|\mathcal{F}_{t-1}\right]\right].
\end{align*}
Then, by using Lemma~\ref{lem:31}, for each $t =1, \ldots, T$,
\begin{align*}
    \mathbb{E}\left[\exp\left(\overline{\lambda} \xi_t\right)|\mathcal{F}_{t-1}\right]&\leq 1 + \frac{\overline{\lambda}^2}{2} \mathbb{E}\left[\xi^2_t| \mathcal{F}_{t-1}\right]  + \sum^\infty_{k=3}
    \frac{\overline{\lambda}^k\mathbb{E}\left[\xi^k_t| \mathcal{F}_{t-1}\right]}{k!}\\
    &\leq 1 + \frac{\overline{\lambda}^2}{2} \mathbb{E}\left[\xi^2_t| \mathcal{F}_{t-1}\right]  + \sum^\infty_{k=3}\overline{\lambda}^k C_1(C_0\sqrt{T})^{-k}\\
    &\leq 1 + \frac{\overline{\lambda}^2}{2} \mathbb{E}\left[\xi^2_t| \mathcal{F}_{t-1}\right]  + O\left(\overline{\lambda}^3/T^{3/2}\right).
\end{align*}
Therefore,
\begin{align*}
    \mathbb{E}\left[\exp\left(\overline{\lambda}\sum^T_{t=1} \xi_t\right)\right] &\leq \mathbb{E}\left[\prod^T_{t=1}\left(1 + \frac{\overline{\lambda}^2}{2} \mathbb{E}\left[\xi^2_t| \mathcal{F}_{t-1}\right] +  O\left(\overline{\lambda}^3/T^{3/2}\right)\right)\right]
    \\
    & \leq \mathbb{E}\left[\prod^T_{t=1}\exp\left(\frac{\overline{\lambda}^2}{2} \mathbb{E}\left[\xi^2_t| \mathcal{F}_{t-1}\right] +  O\left(\overline{\lambda}^3/T^{3/2}\right)\right)\right].
\end{align*}
Similarly, by using Lemma~\ref{lem:31} and constants $c, \tilde{c}>0$, we have
\begin{align*}
    \mathbb{E}\left[\exp\left(\overline{\lambda} \xi_t\right)\right]
    &= \exp\left(\log \mathbb{E}\left[\exp\left(\overline{\lambda} \xi_t\right)\right]\right)
    \\
    &= \exp\left(\log \left(1 +  \sum^\infty_{k=2}\mathbb{E}\left[\frac{(\overline{\lambda}\xi_t)^k}{k!}\right]\right)\right)
    \\
    &= \exp\left(\frac{\overline{\lambda}^2}{2} \mathbb{E}\left[\xi^2_t\right] + \sum^\infty_{k=3}\mathbb{E}\left[\frac{(\overline{\lambda}\xi_t)^k}{k!}\right] - \frac{1}{2}\left( \sum^\infty_{k=2}\mathbb{E}\left[\frac{(\overline{\lambda}\xi_t)^k}{k!}\right]\right)^2 + \frac{1}{3}\left( \sum^\infty_{k=2}\mathbb{E}\left[\frac{(\overline{\lambda}\xi_t)^k}{k!}\right]\right)^3 + \cdots\right)
    \\
     & \stackrel{(a)}{\geq} \exp\left(\frac{\overline{\lambda}^2}{2} \mathbb{E}\left[\xi^2_t\right] - \sum^\infty_{k=3}\mathbb{E}\left[\frac{|\overline{\lambda}\xi_t|^k}{k!}\right] - \frac{1}{2}\left( \sum^\infty_{k=2}\mathbb{E}\left[\frac{|\overline{\lambda}\xi_t|^k}{k!}\right]\right)^2 - \frac{1}{3}\left( \sum^\infty_{k=2}\mathbb{E}\left[\frac{|\overline{\lambda}\xi_t|^k}{k!}\right]\right)^3 + \cdots\right)
     \\
     & \stackrel{(b)}{\ge} \exp\left(\frac{\overline{\lambda}^2}{2} \mathbb{E}\left[\xi^2_t\right] - c\overline{\lambda}^3/T^{3/2} - \frac{1}{2}\left( \frac{4 C_1\overline{\lambda}^2}{3 C_0^2 T}\right)^2 - \frac{1}{3}\left( \frac{4 C_1\overline{\lambda}^2}{3 C_0^2 T}\right)^3 - \frac{1}{4}\left( \frac{4 C_1 \overline{\lambda}^2}{3 C_0^2 T}\right)^4 -  \cdots\right)
     \\
     & \ge \exp\left(\frac{\overline{\lambda}^2}{2} \mathbb{E}\left[\xi^2_t\right] - c\overline{\lambda}^3/T^{3/2} - \left( \frac{4 C_1\overline{\lambda}^2}{3 C_0^2 T}\right)^2 - \left( \frac{4 C_1\overline{\lambda}^2}{3 C_0^2 T}\right)^3 - \left( \frac{4 C_1\overline{\lambda}^2}{3 C_0^2 T}\right)^4 - \cdots\right)
     \\
     & \ge \exp\left(\frac{\overline{\lambda}^2}{2} \mathbb{E}\left[\xi^2_t\right] - c\overline{\lambda}^3/T^{3/2} - \left(\frac{4 C_1\overline{\lambda}^2}{3 C_0^2 T}\right)^2\frac{1}{1- \frac{1}{2}}\right)
     \\
     & \stackrel{(c)}{\geq} \exp\left(\frac{\overline{\lambda}^2}{2} \mathbb{E}\left[\xi^2_t\right] - c \left(\overline{\lambda}^3/\sqrt{T}\right)^3- \tilde{c} \overline{\lambda}^4/T^2\right).
\end{align*}
For $(a)$, we used Jensen's inequality for $m = 2,3,\dots$ as
\begin{align*}
 - (-1)^{m} \frac{1}{m}\left( \sum^\infty_{k=2}\mathbb{E}\left[\frac{(\overline{\lambda}\xi_t)^k}{k!}\right]\right)^m \geq - \frac{1}{m}\left( \sum^\infty_{k=2}\left|\mathbb{E}\left[\frac{(\overline{\lambda}\xi_t)^k}{k!}\right]\right|\right)^m \ge   - \frac{1}{m}\left(\sum^\infty_{k=2}\mathbb{E}\left[\frac{|\overline{\lambda}\xi_t|^k}{k!}\right]\right)^m.
\end{align*}
For $(b)$, we used the fact there exist a constant $c>0$ such that 
\begin{align}
    \Ebb \left[\sum^\infty_{k=2}\frac{|\overline{\lambda}\xi_t|^k}{k!} \right]& \stackrel{(c)}{\le}  \sum^\infty_{k=2}\frac{\overline{\lambda}^k}{k!} \cdot k! C_1 \frac{1}{(C_0 \sqrt{T})^k}  = C_1 \sum^\infty_{k=2} \left(\frac{\overline{\lambda}}{C_0\sqrt{T}}\right)^k = \frac{C_1\overline{\lambda}^2}{C_0^2T} \frac{1}{1 - \frac{\overline{\lambda}}{C_0 \sqrt{T}}}
    \nonumber\\
    & \stackrel{(d)}{\le}  \frac{ C_1\overline{\lambda}^2}{C_0^2T}  \frac{1}{1 - \frac{1}{4}} = \frac{4 C_1\overline{\lambda}^2}{3 C_0^2 T}
 \stackrel{(d)}{\le} \frac{1}{2}\nonumber,
\end{align}
and
\begin{align*}
    \Ebb\left[\sum^\infty_{k=3}\frac{|\overline{\lambda}\xi_t|^k}{k!} \right]&  \stackrel{(c)}{\le} \sum^\infty_{k=3}\frac{\overline{\lambda}^k}{k!} \cdot k! C_1 \frac{1}{(C_0 \sqrt{T})^k} \le c \left(\frac{\overline{\lambda}}{\sqrt{T}} \right)^3,
\end{align*}
for $(c)$, we used Lemma~\ref{lem:31}, and for $(d)$, we used \eqref{eq:34}.
Then, by combining the above upper and lower bounds, with some constant $\tilde{C}_0, \tilde{C}_1>0$,
\begin{align}
\frac{\mathbb{E}\left[\exp\left(\overline{\lambda}\sum^T_{t=1} \xi_t\right)\right]}{\prod^T_{t=1}\mathbb{E}\left[\exp\left(\overline{\lambda} \xi_t\right)\right]}&\leq \frac{ \mathbb{E}\left[\prod^T_{t=1}\exp\left(\frac{\overline{\lambda}^2}{2} \mathbb{E}\left[\xi^2_t| \mathcal{F}_{t-1}\right] + O\left( \left(\overline{\lambda}/\sqrt{T} \right)^3\right)\right)\right]}{\prod^T_{t=1}\exp\left(\frac{\overline{\lambda}^2}{2} \mathbb{E}\left[\xi^2_t\right]  - c \left(\overline{\lambda}^3/\sqrt{T}\right)^3- \tilde{c} \overline{\lambda}^4/T^2 \right)}\nonumber
\\
\label{eq:upper_bound_fraction}
&= \exp\left(\tilde{C}_0\overline{\lambda}^4/T + \tilde{C}_1\overline{\lambda}^3/\sqrt{T} \right)\mathbb{E}\left[\prod^T_{t=1}\exp\left(\overline{\lambda}^2\left(\mathbb{E}[\xi^2_t| \mathcal{F}_{t-1}] - \mathbb{E}[\xi^2_t]\right)/2 \right)\right].
\end{align}
Using Hölder's inequality,
\begin{align}
\frac{\mathbb{E}\left[\exp\left(\overline{\lambda}\sum^T_{t=1} \xi_t\right)\right]}{\prod^T_{t=1}\mathbb{E}\left[\exp\left(\overline{\lambda} \xi_t\right)\right]}&\leq \exp\left(\tilde{C}_0\overline{\lambda}^4/T + \tilde{C}_1\overline{\lambda}^3/\sqrt{T} \right)\mathbb{E}\left[\prod^T_{t=1}\exp\left(\overline{\lambda}^2\left(\mathbb{E}[\xi^2_t| \mathcal{F}_{t-1}] - \mathbb{E}[\xi^2_t]\right)/2 \right)\right]
\nonumber\\
\label{eq:jensen}
&\leq \exp\left(\tilde{C}_0\overline{\lambda}^4/T + \tilde{C}_1\overline{\lambda}^3/\sqrt{T} \right) \prod^T_{t=1} \left(\mathbb{E}\left[\exp\left(T\overline{\lambda}^2\left(\mathbb{E}[\xi^2_t| \mathcal{F}_{t-1}] - \mathbb{E}[\xi^2_t]\right)/2 \right)\right] \right)^{\frac{1}{T}}.
\end{align}
Note that the term
\begin{align*}
\frac{\overline{\lambda}^2}{2}\left(\mathbb{E}[\xi^2_t| \mathcal{F}_{t-1}] - \mathbb{E}[\xi^2_t]\right)
&=\frac{\overline{\lambda}^2}{2T}\left(\mathbb{E}\left[\left(\hat{X}_{1, t} - \hat{X}_{a, t}- \Gap\right)^2\Big| \mathcal{F}_{t-1}\right] - \mathbb{E}\left[\left(\hat{X}_{1, t} - \hat{X}_{a, t}- \Gap\right)^2 \right]\right)
\end{align*}
is bounded by some constant because $w_{a,t}$ and $\hat{\mu}_{a,t}$ are bounded and $\overline{\lambda} \le \sqrt{T}\min\left\{\frac{1}{4} C_0, \sqrt{\frac{3 C_0^2}{8 C_1}}\right\}$. Then, Lemma~\ref{lem:as_conv} and Proposition~\ref{prp:lr_conv_theorem}, with probability one, as $t\to \infty$, 
\begin{align*}
    & t^\alpha\left|\mathbb{E}\left[\left(\hat{X}_{1, t} - \hat{X}_{a, t}- \Gap\right)^2\Big| \mathcal{F}_{t-1}\right] - \mathbb{E}\left[\left(\hat{X}_{1, t} - \hat{X}_{a, t}- \Gap\right)^2 \right]\right| 
    \\
    & \leq t^\alpha\left|\mathbb{E}\left[\left(\hat{X}_{1, t} - \hat{X}_{a, t}- \Gap\right)^2\Big| \mathcal{F}_{t-1}\right] -  \left(\frac{\sigma^2_1}{w_1^*} +\frac{\sigma^2_0}{w_a^*} \right)\right| +  t^\alpha\left|\mathbb{E}\left[\left(\hat{X}_{1, t} - \hat{X}_{a, t}- \Gap\right)^2\right] -  \left(\frac{\sigma^2_1}{w_1^*} +\frac{\sigma^2_0}{w_a^*} \right)\right| 
    \\
    & = t^\alpha\left|\mathbb{E}\left[\left(\hat{X}_{1, t} - \hat{X}_{a, t}- \Gap\right)^2\Big| \mathcal{F}_{t-1}\right] -  \left(\frac{\sigma^2_1}{w_1^*} +\frac{\sigma^2_0}{w_a^*} \right)\right|
    \\
    & \ \ \ +  t^\alpha\left|\mathbb{E}\left[\mathbb{E}\left[\left(\hat{X}_{1, t} - \hat{X}_{0, t}- \Gap\right)^2 \middle| \set{F}_{t-1}\right] -  \left(\frac{\sigma^2_1}{w_1^*} +\frac{\sigma^2_0}{w_a^*} \right)\right]\right| 
    \\
    & \to  0,
\end{align*}
where we used the boundedness of $\mathbb{E}\left[\left(\hat{X}_{1, t} - \hat{X}_{a, t}- \Gap\right)^2\Big| \mathcal{F}_{t-1}\right]$ to derive the mean convergence.
Here, let us define an event $\mathcal{E}$ such that 
\begin{align*}
    \mathcal{E} &= \left\{t^{\alpha}\left|\mathbb{E}\left[\left(\hat{X}_{1, t} - \hat{X}_{a, t}- \Gap\right)^2\Big| \mathcal{F}_{t-1}\right] - \mathbb{E}\left[\left(\hat{X}_{1, t} - \hat{X}_{a, t}- \Gap\right)^2 \right]\right| \to 0\ \mathrm{as}\ t \to \infty\right\}\\ 
    &= \left\{Tt^{\alpha}\left|\left(\mathbb{E}[\xi^2_t| \mathcal{F}_{t-1}] - \mathbb{E}[\xi^2_t]\right)\right| \to 0\ \mathrm{as}\ t\to \infty\right\},
\end{align*}
which occurs with probability one, where $\alpha > 0$ is some constant. Note that the second equality holds by definition. Without loss of generality, we assume that $\alpha \neq 1$. On the event $\mathcal{E}$, for all $\varepsilon > 0$, there exists $T_0 \geq 0$ such that for all $t > T_0$,
\begin{align*}
    \exp\left(T\overline{\lambda}^2\left(\mathbb{E}[\xi^2_t| \mathcal{F}_{t-1}] - \mathbb{E}[\xi^2_t]\right)/2 \right) \leq \exp\left(\overline{\lambda}^2\varepsilon / t^\alpha \right).
\end{align*}
Because this event occurs with probability one and $\mathbb{E}[\xi^2_t| \mathcal{F}_{t-1}] - \mathbb{E}[\xi^2_t]$ is bounded, for all $\varepsilon > 0$, there exists $T_0$ such that for all $t > T_0$,
\begin{align*}
    \mathbb{E}\left[\exp\left(T\overline{\lambda}^2\left(\mathbb{E}[\xi^2_t| \mathcal{F}_{t-1}] - \mathbb{E}[\xi^2_t]\right)/2 \right)\right] \leq \exp\left(\overline{\lambda}^2\varepsilon / t^\alpha \right). 
\end{align*}
From this result, for all $t > T_0$,
\begin{align*}
    \mathbb{E}\left[\exp\left(T\overline{\lambda}^2\left(\mathbb{E}[\xi^2_t| \mathcal{F}_{t-1}] - \mathbb{E}[\xi^2_t]\right)/2 \right)\right]^{1/T} \leq \exp\left(\overline{\lambda}^2\varepsilon/(Tt^\alpha) \right). 
\end{align*}
Therefore, in \eqref{eq:jensen}, from the boundedness of the random variables, for a constant $C>0$
\begin{align*}
&\prod^T_{t=1} \left(\mathbb{E}\left[\exp\left(T\overline{\lambda}^2\left(\mathbb{E}[\xi^2_t| \mathcal{F}_{t-1}] - \mathbb{E}[\xi^2_t]\right)/2 \right)\right] \right)^{\frac{1}{T}}\\
&\leq \prod^T_{t=T_0 + 1}\exp\left(\overline{\lambda}^2\varepsilon/Tt^{\alpha}  \right)\prod^{T_0}_{t=1} \mathbb{E}\left[\exp\left(T\overline{\lambda}^2\left(\mathbb{E}[\xi^2_t| \mathcal{F}_{t-1}] - \mathbb{E}[\xi^2_t]\right)/2 \right)\right]^{1/T}\\
&\leq \prod^T_{t=T_0 + 1}\exp\left(\overline{\lambda}^2\varepsilon/Tt^{\alpha}  \right)\prod^{T_0}_{t=1} \exp(C)\\
&\leq \exp\left(\overline{\lambda}^2\varepsilon  T^{-\alpha} / (1 - \alpha) + CT_0 \right).
\end{align*}
In summary, for any $\varepsilon > 0$ and some constants $\tilde{C}_2, \tilde{C}_3, \tilde{C}_4>0$, there exists $T_0 > 0$ such that for all $T \geq T_0$, 
\begin{align*}
\frac{\mathbb{E}\left[\exp\left(\overline{\lambda}\sum^T_{t=1} \xi_t\right)\right]}{\prod^T_{t=1}\mathbb{E}\left[\exp\left(\overline{\lambda} \xi_t\right)\right]} 
&\leq \exp\left(\tilde{C}_2\overline{\lambda}^4/T + \tilde{C}_3\overline{\lambda}^3/\sqrt{T} + \tilde{C}_4T_0 + \varepsilon \overline{\lambda}^2  T^{-\alpha} / (1 - \alpha)\right).
\end{align*}
\end{proof}

By using Lemmas~\ref{lem:31}--\ref{lem:34}, we show the proof of Theorem~\ref{thm:fan_refine}.
\begin{proof}[Proof of Theorem~\ref{thm:fan_refine}]  
There exists some constant $C>0$ such that for all $1\leq u \leq \sqrt{T}\min\left\{\frac{1}{4} C_0, \sqrt{\frac{3 C_0^2}{8 C_1}}\right\}$,
\begin{align*}
&\mathbb{P}\left(Z_T > u\right)
\\
&=
\int\left(\prod^T_{t=1}\frac{\exp\left(\lambda \xi_t\right)}{\mathbb{E}\left[\exp\left(\lambda \xi_t\right)\right]}\right)\left(\prod^T_{t=1}\frac{\exp\left(\lambda \xi_t\right)}{\mathbb{E}\left[\exp\left(\lambda \xi_t\right)\right]}\right)^{-1}\mathbbm{1}[Z_T > u]\mathrm{d}\mathbb{P}
\\
&=
\int\left(\prod^T_{t=1}\frac{\exp\left(\lambda \xi_t\right)}{\mathbb{E}\left[\exp\left(\lambda \xi_t\right)\right]}\right) \exp\left(-\lambda \sum_{t=1}^T\xi_t + \log \left(\prod_{t=1}^T\Ebb[\exp(\lambda \xi_t)]\right)\right)\mathbbm{1}[Z_T > u]\mathrm{d}\mathbb{P}
\\
\\
&=
\int\left(\prod^T_{t=1}\frac{\exp\left(\lambda \xi_t\right)}{\mathbb{E}\left[\exp\left(\lambda \xi_t\right)\right]}\right)\exp\left(-\lambda Z_T + \Psi_T(\lambda)\right)\mathbbm{1}[Z_T > u]\mathrm{d}\mathbb{P}
\\
&=
\int\left(\prod^T_{t=1}\frac{\exp\left(\lambda \xi_t\right)}{\mathbb{E}\left[\exp\left(\lambda \xi_t\right)\right]}\right)\exp\left(-\lambda U_T(\lambda) - \lambda B_T(\lambda) + \Psi_T(\lambda)\right)\mathbbm{1}[U_T(\lambda) + B_T(\lambda) > u]\mathrm{d}\mathbb{P},
\\
& \le \int\left(\prod^T_{t=1}\frac{\exp\left(\lambda \xi_t\right)}{\mathbb{E}\left[\exp\left(\lambda \xi_t\right)\right]}\right)\exp\left(-\lambda U_T(\lambda)  -\frac{\lambda^2}{2} + C(\lambda^3 /\sqrt{T}+ \lambda^2 V_T)\right)
\\
& \qquad \qquad \qquad \cdot \mathbbm{1} \left[U_T(\lambda) + \lambda + C(\lambda V_T + \lambda^2 /\sqrt{T}) > u\right]\mathrm{d}\mathbb{P},
\end{align*}
where for the last inequality, we used Lemma~\ref{lem:32} and Lemma~\ref{lem:33}.
Let $\overline{\lambda} = \overline{\lambda}(u)$ be the largest solution of the equation
\begin{align*}
    \lambda + C(\lambda V_T + \lambda^2 /\sqrt{T}) = u.
\end{align*}
The definition of $\overline{\lambda}$ implies that there exist $C' > 0$ such that, for all $1\leq u \leq \sqrt{T}\min\left\{\frac{1}{4} C_0, \sqrt{\frac{3 C_0^2}{8 C_1}}\right\}$,
\begin{align}
\label{eq:34}
    C' u \leq \overline{\lambda}(u) = \frac{2u}{\sqrt{(1+CV_T)^2 + 4C u /\sqrt{T}} + C V_T +1} \leq u
\end{align}
and there exists $\theta \in (0, 1]$ such that 
\begin{align}
    \overline{\lambda}(u) & = u - C  (\overline{\lambda} V_T + \overline{\lambda}^2/\sqrt{T}) 
    \nonumber\\
    \label{eq:35}
    & = u - C \theta (u V_T + u^2 /\sqrt{T}) \in \left[C', \sqrt{T}\min\left\{\frac{1}{4} C_0, \sqrt{\frac{3 C_0^2}{8 C_1}}\right\}\right].
\end{align}
Then, we obtain for all $1\leq u \leq \sqrt{T}\min\left\{\frac{1}{4} C_0, \sqrt{\frac{3 C_0^2}{8 C_1}}\right\}$,
\begin{align*}
    & \mathbb{P}\left(Z_T > u\right) 
    \\
    & \leq \exp\left(C\left(\overline{\lambda}^3T^{-1/2} + \overline{\lambda}^2 V_T \right) - \overline{\lambda}^2/2\right)\int\left(\prod^T_{t=1}\frac{\exp\left(\overline{\lambda} \xi_t\right)}{\mathbb{E}\left[\exp\left(\overline{\lambda} \xi_t\right)\right]}\right)\exp\left(-\overline{\lambda} U_T(\overline{\lambda})\right)\mathbbm{1}[U_T(\overline{\lambda}) > 0]\mathrm{d}\mathbb{P}.
\end{align*}
Here, we have
\begin{align*}
    &\int\left(\prod^T_{t=1}\frac{\exp\left(\overline{\lambda} \xi_t\right)}{\mathbb{E}\left[\exp\left(\overline{\lambda} \xi_t\right)\right]}\right)\exp\left(-\overline{\lambda} U_T(\overline{\lambda})\right)\mathbbm{1}[U_T(\overline{\lambda}) > 0]\mathrm{d}\mathbb{P}\\
    &=\mathbb{E}\left[ \prod^T_{t=1}\frac{\exp\left(\overline{\lambda} \xi_t\right)}{\mathbb{E}\left[\exp\left(\overline{\lambda} \xi_t\right)\right]}\exp\left(-\overline{\lambda} U_T(\overline{\lambda})\right)\mathbbm{1}[U_T(\overline{\lambda}) > 0]\right].
\end{align*}
We also define another measure $\widetilde{\mathbb{P}}_\lambda $ as
\begin{align*}
    \mathrm{d}\widetilde{\mathbb{P}}_\lambda = \frac{\prod^T_{t=1}\exp\left(\lambda \xi_t\right)}{\mathbb{E}\left[\exp\left(\lambda\sum^T_{t=1} \xi_t\right)\right]}\mathrm{d}\mathbb{P}= \frac{\exp\left(\lambda \sum^T_{t=1}\xi_t\right)}{\mathbb{E}\left[\exp\left(\lambda\sum^T_{t=1} \xi_t\right)\right]}\mathrm{d}\mathbb{P}.
\end{align*}
Note that $\widetilde{\mathbb{P}}_\lambda $ is a probability measure, as the following holds
\begin{align*}
\int \mathrm{d}\widetilde{\mathbb{P}}_\lambda & = \int \frac{\exp\left(\lambda \sum^T_{t=1}\xi_t\right)}{\mathbb{E}\left[\exp\left(\lambda\sum^T_{t=1} \xi_t\right)\right]}\mathrm{d}\mathbb{P}
\\
& = \frac{1}{\mathbb{E}\left[\exp\left(\lambda\sum^T_{t=1} \xi_t\right)\right]} \int \exp\left(\lambda \sum^T_{t=1}\xi_t\right)  \mathrm{d}\mathbb{P}
\\
& = \frac{1}{\mathbb{E}\left[\exp\left(\lambda\sum^T_{t=1} \xi_t\right)\right]} \cdot \mathbb{E}\left[\exp\left(\lambda\sum^T_{t=1} \xi_t\right)\right]
\\
& = 1.
\end{align*}
We further denote $\widetilde{\mathbb{E}}_\lambda $ as the expectation under the measure $\widetilde{\mathbb{P}}_\lambda $.
In the same way as (37) and (38) in \citet{Fan2013}, it is easy to see that
\begin{align}
    &\mathbb{E}\left[ \prod^T_{t=1}\frac{\exp\left(\overline{\lambda} \xi_t\right)}{\mathbb{E}\left[\exp\left(\overline{\lambda} \xi_t\right)\right]}\exp\left(-\overline{\lambda} U_T(\overline{\lambda})\right)\mathbbm{1}[U_T(\overline{\lambda}) > 0]\right]
    \nonumber\\
    & = \frac{\mathbb{E}[\exp(\overline{\lambda}\sum_{t=1}^T \xi_t)]}{\prod^T_{t=1} \mathbb{E}\left[\exp\left(\overline{\lambda} \xi_t\right)\right]} 
    \mathbb{E}\left[ \frac{ \prod^T_{t=1}\exp\left(\overline{\lambda} \xi_t\right)}{\mathbb{E}[\exp(\overline{\lambda}\sum_{t=1}^T \xi_t)]}\exp\left(-\overline{\lambda} U_T(\overline{\lambda})\right)\mathbbm{1}[U_T(\overline{\lambda}) > 0]\right]
    \nonumber\\
    & = \frac{\mathbb{E}[\exp(\overline{\lambda}\sum_{t=1}^T \xi_t)]}{\prod^T_{t=1} \mathbb{E}\left[\exp\left(\overline{\lambda} \xi_t\right)\right]} \widetilde{\mathbb{E}}_{\overline{\lambda}}[\exp\left(-\overline{\lambda} U_T(\overline{\lambda})\right)\mathbbm{1}[U_T(\overline{\lambda}) > 0]]
    \nonumber\\
    & = \frac{\mathbb{E}\left[\exp\left(\lambda\sum^T_{t=1} \xi_t\right)\right]}{\prod^T_{t=1}\mathbb{E}\left[\exp\left(\lambda \xi_t\right)\right]} \int^\infty_0\overline{\lambda} \exp(-\overline{\lambda}y)\widetilde{\mathbb{P}}_{\overline{\lambda}}(0< U_T(\overline{\lambda}) < y) \mathrm{d}y, \label{eq:fans_37}.
\end{align}
Besides, for a standard Gaussian random variable $\mathcal{N}$, 
\begin{align}
        &\mathbb{E}\left[\exp\left(-\overline{\lambda} \mathcal{N}\right)\mathbbm{1}[\mathcal{N} > 0]\right] = \int^\infty_0\overline{\lambda} \exp(-\overline{\lambda}y)\mathbb{P}(0< \mathcal{N} < y) \mathrm{d}y.\label{eq:fans_38}
\end{align}
Then, from \eqref{eq:fans_37} and \eqref{eq:fans_38}, 
\begin{align*}
    \left| \widetilde{\mathbb{E}}_{\overline{\lambda}}[\exp\left(-\overline{\lambda} U_T(\overline{\lambda})\right)\mathbbm{1}[U_T(\overline{\lambda}) > 0]] - \mathbb{E}\left[\exp\left(-\overline{\lambda} \mathcal{N}\right)\mathbbm{1}[\mathcal{N} > 0]\right]\right| \le 2 \sup_{g} \left|\widetilde{\mathbb{P}}_{\overline{\lambda}}\left(U_T(\overline{\lambda})\leq g\right) - \Phi(g)\right|
\end{align*}
Therefore,
\begin{align*}
    &\mathbb{P}\left(Z_T > u\right)
    \\
    & \le \frac{\mathbb{E}\left[\exp\left(\lambda\sum^T_{t=1} \xi_t\right)\right]}{\prod^T_{t=1}\mathbb{E}\left[\exp\left(\lambda \xi_t\right)\right]} \exp\left(C\left(\overline{\lambda}^3/\sqrt{T} + \overline{\lambda}^2 V_T \right) - \overline{\lambda}^2/2\right)
    \widetilde{\mathbb{E}}_{\overline{\lambda}}[\exp\left(-\overline{\lambda} U_T(\overline{\lambda})\right)\mathbbm{1}[U_T(\overline{\lambda}) > 0]]
    \\
    &\leq \frac{\mathbb{E}\left[\exp\left(\lambda\sum^T_{t=1} \xi_t\right)\right]}{\prod^T_{t=1}\mathbb{E}\left[\exp\left(\lambda \xi_t\right)\right]} \exp\left(C\left(\overline{\lambda}^3/\sqrt{T} + \overline{\lambda}^2 V_T \right) - \overline{\lambda}^2/2\right) 
    \\
    & \qquad \qquad \qquad \times \left(\mathbb{E}\left[\exp\left(-\overline{\lambda} \mathcal{N}\right)\mathbbm{1}[\mathcal{N} > 0]\right] + 2\sup_{g}\left|\widetilde{\mathbb{P}}_{\overline{\lambda}}\left(U_T(\overline{\lambda})\leq g\right) - \Phi(g)\right|\right)
    \\
    &\leq \frac{\mathbb{E}\left[\exp\left(\lambda\sum^T_{t=1} \xi_t\right)\right]}{\prod^T_{t=1}\mathbb{E}\left[\exp\left(\lambda \xi_t\right)\right]}\exp\left(C\left(\overline{\lambda}^3/\sqrt{T} + \overline{\lambda}^2 V_T \right) - \overline{\lambda}^2/2\right)\left(\mathbb{E}\left[\exp\left(-\overline{\lambda} \mathcal{N}\right)\mathbbm{1}[\mathcal{N} > 0]\right] + 2 \right).
\end{align*}
Here,
\begin{align*}
\exp\left(- \overline{\lambda}^2/2\right)\mathbb{E}\left[\exp\left(-\overline{\lambda} \mathcal{N}\right)\mathbbm{1}[\mathcal{N} > 0]\right] = \frac{1}{\sqrt{2\pi}} \int^\infty_0 \exp\left(-(y+ \overline{\lambda})^2\right)\mathrm{d}y = 1 - \Phi(\overline{\lambda}).
\end{align*}
From (41) of \citet{Fan2013}, for all $\overline{\lambda} \ge C'$, we have
\begin{align*}
\frac{C'}{\sqrt{2 \pi}(1 + C')}\frac{1}{\overline{\lambda}}\exp\left( - \frac{\overline{\lambda}^2}{2}\right) \leq 1 - \Phi(\overline{\lambda}).
\end{align*}
Therefore, with some constant $\tilde{C}$, for all $1\leq u \leq \sqrt{T}\min\left\{\frac{1}{4} C_0, \sqrt{\frac{3 C_0^2}{8 C_1}}\right\}$,
\begin{align}
    \mathbb{P}\left({Z}_T > u\right) &\leq\frac{\mathbb{E}\left[\exp\left(\lambda\sum^T_{t=1} \xi_t\right)\right]}{\prod^T_{t=1}\mathbb{E}\left[\exp\left(\lambda \xi_t\right)\right]} \left\{
    \Big(1-\Phi(\overline{\lambda})\Big) + \overline{\lambda} \Big(1-\Phi(\overline{\lambda})\Big) c \right\}\exp\left(C\left(\overline{\lambda}^3/\sqrt{T} + \overline{\lambda}^2 V_T \right)\right)
   \nonumber\\
    &\leq \frac{\mathbb{E}\left[\exp\left(\lambda\sum^T_{t=1} \xi_t\right)\right]}{\prod^T_{t=1}\mathbb{E}\left[\exp\left(\lambda \xi_t\right)\right]}
    \tilde{C}\overline{\lambda}
    \Big(1-\Phi(\overline{\lambda})\Big)\exp\left(C\left(\overline{\lambda}^3/\sqrt{T} + \overline{\lambda}^2 V_T \right)\right),\label{eq:upper_bound_wo_frac}
\end{align}
where $c = \sqrt{2\pi} (1+C') / C'$, and $\tilde{C}$ is chosen to be $\tilde{C} \overline{\lambda} \geq (1 + \overline{\lambda} c)$ (Note that $\overline{\lambda} \ge C'$ from \eqref{eq:34}).
Now, from Lemma~\ref{lem:34}, for $\varepsilon > 0$, there exists $T_0 > 0$ such that for all $T \geq T_0$, 
\begin{align}
\label{eq:upper_bound_frac_fin}
\frac{\mathbb{E}\left[\exp\left(\overline{\lambda}\sum^T_{t=1} \xi_t\right)\right]}{\prod^T_{t=1}\mathbb{E}\left[\exp\left(\overline{\lambda} \xi_t\right)\right]} 
&\leq \exp\left(\tilde{C}_2\overline{\lambda}^4/T + \tilde{C}_3\overline{\lambda}^3/\sqrt{T} + \tilde{C}_4T_0 + \varepsilon \overline{\lambda}^2 \right).
\end{align}

In summary, by \eqref{eq:upper_bound_wo_frac} and \eqref{eq:upper_bound_frac_fin}, for all $1\leq u \leq \sqrt{T}\min\left\{\frac{1}{4} C_0, \sqrt{\frac{3 C_0^2}{8 C_1}}\right\}$,
\begin{align}
\label{eq:42}
    \frac{\mathbb{P}\left({Z}_T > u\right)}{1-\Phi(\overline{\lambda})} &\leq  \tilde{C}\overline{\lambda} \exp\left(\tilde{C}_2\overline{\lambda}^4/T + \tilde{C}_3\overline{\lambda}^3/\sqrt{T}  + C\left(\overline{\lambda}^3/\sqrt{T} + \overline{\lambda}^2 V_T + T_0\right) + \varepsilon \overline{\lambda}^2\right).
\end{align}
Next, we compare $1-\Phi(\overline{\lambda})$ with $1-\Phi(u)$.
Recall the following upper bound and lower bound on $1 - \Phi(x) = \Phi(-x)$:
\begin{align*}
    \frac{1}{\sqrt{2 \pi} (1+x)} \exp\left(-\frac{x^2}{2}\right) \le \Phi(-x) \le \frac{1}{\sqrt{ \pi} (1+x)} \exp\left(-\frac{x^2}{2}\right), \; x\ge 0.
\end{align*}
For all $1\leq u \leq \sqrt{T}\min\left\{\frac{1}{4} C_0, \sqrt{\frac{3 C_0^2}{8 C_1}}\right\}$,
\begin{align*}
1 & \leq \frac{\int^\infty_{\overline{\lambda}}\exp(-t^2/2)\mathrm{d}t}{\int^\infty_u \exp(-t^2/2)\mathrm{d}t} 
\\
& \le \frac{\frac{1}{\sqrt{\pi}(1+ \overline{\lambda})} \exp(- \overline{\lambda}^2/2)}{\frac{1}{\sqrt{2\pi}(1+ u)} \exp(- u^2/2)}
\\
& = \sqrt{2}\frac{1+u}{1+ \overline{\lambda}}\exp((u^2 - \overline{\lambda}^2)/2).
\end{align*}
From \eqref{eq:35}, we have
\begin{align*}
    u^2 - \overline{\lambda}^2 & = (u + \overline{\lambda})(u - \overline{\lambda})
    \\
    &  \le  2u  (C \theta (u V_T + u^2 /\sqrt{T}))
    \\
    &  = 2C \theta (u^2 V_T + u^3 /\sqrt{T}).
\end{align*}
Therefore, with some constant $\tilde{C}_4 > 0$
\begin{align*}
\frac{\int^\infty_{\overline{\lambda}}\exp(-t^2/2)\mathrm{d}t}{\int^\infty_u \exp(-t^2/2)\mathrm{d}t} 
&\leq \exp\left(\tilde{C}_4 \left(u^2 V_T + u^3 /\sqrt{T}\right)\right).
\end{align*}
We find that
\begin{align}
\label{eq:44}
1 - \Phi(\overline{\lambda}) \le \big(1 - \Phi(u)\big)\exp\left( \tilde{C}_4\left(u^2 V_T + u^3 /\sqrt{T}\right)\right).
\end{align}
By combining \eqref{eq:42}, \eqref{eq:44}, and \eqref{eq:34}, for $\varepsilon > 0$ all $1\leq u \leq \sqrt{T}\min\left\{\frac{1}{4} C_0, \sqrt{\frac{3 C_0^2}{8 C_1}}\right\}$, there exist constants $T_0 > 0$ and $\tilde{C}_5>0$ such that for all $T \geq T_0$,
\begin{align*}
    &\frac{\mathbb{P}\left({Z}_T > u\right)}{1-\Phi(u)}\\
    &\leq  \tilde{C} \overline{\lambda} \exp\left( C\left(\overline{\lambda}^3/\sqrt{T} + \overline{\lambda}^2 V_T \right) + \tilde{C}_2\overline{\lambda}^4/T + \tilde{C}_3\overline{\lambda}^3/\sqrt{T} + \tilde{C}_4 \left( u^2 V_T + u^3 /\sqrt{T}+ T_0\right) + \varepsilon u^2 \right)
    \\
    &\le \tilde{C}u \exp\left( \tilde{C}_5 \left( u^2 V_T + u^3 /\sqrt{T} + u^4/T+ T_0\right) + \varepsilon u^2 \right)\\
    &= \tilde{C}u \exp\left( \tilde{C}_5 \left( u^2 (V_T + \varepsilon) + u^3 /\sqrt{T} + u^4/T+ T_0\right) \right).
\end{align*}
Applying the same argument to the martingale $-Z_T$, we conclude the proof.

\end{proof}

\section{Additional discussions and related work} 
\label{sec:related2}

\subsection{Additional literature on BAI}
The stochastic MAB problem is a classical abstraction of the sequential decision-making problem \citep{Thompson1933,Robbins1952,Lai1985}. BAI is a paradigm of the MAB problem \citep{EvanDar2006,Audibert2010,Bubeck2011}. Though the problem of BAI itself goes back decades, variants go as far back as the 1950s \citet{bechhofer1968sequential}. 

\citet{kaufmann14,Kaufman2016complexity} shows distribution-dependent lower bound for BAI. In the BAI literature, there is another setting called BAI with fixed confidence \citep{Jenninson1982,Mannor2004,Kalyanakrishnan2012,wang2021fast}. For the fixed confidence setting, \citet{Garivier2016} solves the problem in the sense that they provided a strategy whose upper bound matches the distribution-dependent lower bound by solving an optimization problem at each step. The result is further developed by \citet{Degenne2019b} to solve the solution of the two-player game by  the no-regret saddle point algorithm. Besides, \citet{Shang2020} shows the asymptotic optimality of the Top Two Thompson Sampling (TTTS), proposed by \citet{Russo2016} and \citet{Qin2017}, in the fixed confidence setting. \citet{wang2021fast} develops Frank-Wolfe-based Sampling (FWS) to complete the characterization of the complexity of fixed-confidence BAI with various types of structures among the arms. See \cite{wang2021fast} for techniques in the fixed-confidence setting and further comprehensive survey.

In addition to these studies, \citet{Russo2016} proposes the TTTS strategy with another direction of asymptotic optimality. There are other Bayesian strategies for BAI, such as \citet{Qin2017} and \citet{Shang2020}. \citet{Komiyama2021} finds that parameters with a small gap make a significant contribution to the evaluation of the Bayesian simple regret.

In contrast, there is less research on BAI with a fixed budget compared to BAI with fixed confidence. Following \citet{Audibert2010} and \citet{Bubeck2011}, \citet{Gabillon2012} and \citet{Karnin2013} discuss the link between the fixed confidence and fixed budget settings. In some studies, the lower bounds are associated only with the gap $\Delta$ (\citet{Audibert2010,Bubeck2011,Carpentier2016} and \citet{Kaufman2016complexity} (Theorem~16)), rather than information-theoretic formulations in \citet{Kaufman2016complexity} (Theorem~12) and our Proposition~\ref{prp:lowerbound_2arms}. 

In \cite{Karnin2013}, they propose the Sequential Halving (SH) strategy and show the performance guarantee
\begin{align*}
    \exp \left(- \frac{T}{8 (\log K)H_2 }\right),
\end{align*}
where $K > 2$ is the number of the arms in the MAB problems. 
In \cite{faella2020rapidly}, they show that the complexity of the BAI with a fixed budget can be refined by using the variances as well as the gap. They introduce the complexity term
\begin{align*}
    H_{\sigma} = \max_{a \in [K]} \frac{\sigma^2_1 + \sigma^2_a}{\Delta^2_a},
\end{align*}
where $\Delta_a = \mu_1 - \mu_a$ and $\mu_a$ is the expected reward of arm $a$ in the MAB problems. 
and propose the strategy called Variance-Based Rejects (VBR). They prove a performance guarantee of VBR of the form 
\begin{align*}
    \exp \left(- \frac{T}{ 2 K(\log K) H_\sigma}\right).
\end{align*}
Their results suggest that the complexities of \citet{Audibert2010, Karnin2013} have some room for improvement depending on the variance of the distribution.

However, the tightest lower bound and the performance guarantee achieving the distribution-dependent lower bound have been long-standing open problems.

\subsection{Difference between limit experiments frameworks}
\label{app_subsec:diff_limit_dec}
For a parameter $\theta_0\in\mathbb{R}$ and $T$ i.i.d. observations, the limit experiments framework considers local alternatives $\theta = \theta_0 + h/\sqrt{T}$ for a constant $h\in\mathbb{R}$ \citep{Vaart1991,Vaart1998}. Then, we can approximate the statistical experiment by the Gaussian distribution and discuss the optimality of statistical procedures under the approximation. 
\citet{Hirano2009} relates the asymptotic optimality of statistical decision rules \citep{Manski2000,Manski2002,Manski2004,DEHEJIA2005} to the limits of experiments framework. This framework is further applied to policy learning, such as \citet{Kitagawa2018} and \citet{AtheySusan2017EPL}. Although \citet{Hirano2020} insists that these local experiment arguments can be applied to adaptive experimental designs, we consider that they cannot be applied to BAI.

The limit experiments arguments are based on the CLT. On the other hand,
taking the parameter $\theta = \theta_0 + h/\sqrt{T}$ will not give us the distribution-dependent analysis; that is, the instance is not fixed as $T$ grows large. Therefore, a naive application of the information-theoretical result like Proposition~\ref{lem:data_proc_inequality} does not provide a lower bound that is valid for the distribution-dependent analysis.
To match the lower bound of  \citet{Kaufman2016complexity}, we need to consider the large deviation bound, not the CLT. However, the conventional large deviation bound cannot be applied to BAI strategies owing to the adaptive sampling strategy. 

In other words, the limit experiment framework first applies the Gaussian approximation and then discusses the efficiency under the approximation, where efficiency arguments are complete within the Gaussian distribution. On the other hand, we derive the lower bound of an event under the true distribution in our limit decision-making. Then, we approximate it by considering the limit of the gap. Therefore, in limit decision-making, we first consider the optimality for the true distribution and find the optimal strategy in the sense that the upper bound matches the lower bound when the gap goes to zero.

\subsection{Literature on causal inference}
The framework of bandit problems is closely related to the potential outcome framework of \citep{Neyman1923,Rubin1974}. In causal inference, the gap is often referred to as the average treatment effect, and the estimation is studied in this framework.

For estimating the average treatment effect efficiently, \citet{Laan2008TheCA}, \citet{Hahn2011}, \citet{Meehan2018}, \citet{Kato2020adaptive} and \citet{gupta2021efficient} propose adaptive strategies.

The AIPW estimator, which is also referred to as a doubly robust (DR) estimator, plays an important role in treatment effect estimation \citep{Robins1994,hahn1998role,bang2005drestimation,dudik2011doubly,Laan2016onlinetml,Luedtke2016}. The AIPW estimator also plays an important role in double/debiased machine learning literature because it mitigates the convergence rate conditions of the nuisance parameters \citep{ChernozhukovVictor2018Dmlf,Ichimura2022}.

When constructing AIPW estimator with samples obtained from adaptive experiments, including BAI strategies, a typical construction is to use sample splitting and martingales \citep{Laan2008TheCA,hadad2019,Kato2020adaptive,Kato2021adr}.  \citet{Howard2020TimeuniformNN}, \citet{Kato2020adaptive}, and provide non-asymptotic confidence intervals of the AIPW or DR estimator, which do not bound a tail probability in large deviation as ours. The AIPW estimator is also used in the recent bandit literature, mainly in regret minimization \citep{dimakopoulou2021online,Kim2021}. \citet{hadad2019}, \citet{Bibaut2021}, and \citet{Zhan2021} consider the off-policy evaluation using observations obtained from regret minimization algorithms.

There are also some studies on problems related to BAI in economics, such as \citet{Wager2022diffusion}, \citet{adusumilli2022minimax}, and \citet{Armstrong2022}. 

\subsection{Other related work}
Another literature on ordinal optimization has been studied in the operation research community \citep{peng2016myopic, Dohyun2021}, and a modern formulation was established in the 2000s \citep{chen2000,glynn2004large}, in which most of those studies consider the estimation of the optimal sampling rule separately from the probability of misidentification.

\citet{Balsubramani2016} and \citet{Howard2020TimeuniformNN} propose sequential testing using law of iterated logarithms and discuss the optimality of sequential testing based on the arguments of \citet{Jamieson2014}.

In addition to \citet{Fan2013,fan2014generalization}, there are also studies that use martingales for obtaining tight bounds \citep{Cappe2013,Juneja2019,Howard2020TimeuniformNN,Kaufmann2021}, where some of them apply change-of-measure techniques. 

\subsection{Future direction}
As is in BAI with fixed confidence, our future direction is to develop strategies for various settings, such as linear \citep{Hoffman2014,Liang2019,KatzSamuels2020}, combinatorial \citep{Chen2014}, and contextual bandits \citep{Russac2021,Kato2021Role,QinRusso2022,Kato2022semipara}.  A natural question is whether a complete characterization is possible under the large gap.

\subsection{Lower bounds for bandit models parameterized by the mean}
This section extends the results to other distributions besides the Gaussian distribution.
We consider the class of bandit models parameterized by the mean of distributions; that is, $\mathcal{M}^\mu = \{\nu = (\mathcal{P}(\mu_1), \mathcal{P}(\mu_0)): (\mu_1, \mu_0)\in\Theta^2, \mu_1\neq \mu_0\}$, where $\Theta\subset \mathbb{R}$ is the parameter space and $\mathcal{P}(\mu_a)$ denotes a distribution of arm $a$ parameterized by the mean $\mu_a$. We also assume that the distribution $\nu_a$ has a density $f_a(x|\mu_a)$. For the density function $f_{a}(x | \mu_a)$, let $\ell_{a}(\mu | x) = \log f_{a}(x | \mu)$ be a log likelihood function.
We define the Fisher information as $I_{a}(\mu_a) =- \mathbb{E}_{\nu}[\ddot{\ell}_a(\mu_a)] $. We assume the following regularity conditions on the Fisher information $I_{a}(\mu)$. 

\begin{assumption}
\label{asm:reg_cond} Let $\interior{\Theta}$ be the interior of $\Theta$.
For each $ a \in \{1, 0\}$, (i) 
the support $\{x: f_a(x|\mu_a) > 0\}$ does not depend on the parameter $\mu_a$; (ii) if $\mu_a \in \interior{\Theta}$, there exist the Fisher information $I_a(\mu_a)$ and a known constant $C_I > 0$ such that $\max\{1/I_a(\mu_a), I_a(\mu_a)\} < C_I$; (iii) the log likelihood function is thrice differentiable with respect to $\mu_a\in\Theta$; (iv) $\mathbb{E}_{\nu}[\dot{\ell}_a(\mu_a)] = 0$, $\mathbb{E}_{\nu}[\ddot{\ell}_a(\mu_a)] = - I_a(\mu_a)$;  
(v) there exist constants $C, C'>0$ such that for all $\mu_a \in \interior{\Theta}$, there exists a neighborhood $U(\mu_a)\subset \mathbb{R}$ such that for all $\tilde{\mu}\in U(\mu_a)$, $\mathbb{E}_{\nu}[\dddot{\ell_a}(\tilde{\mu})] \leq C$ and $|I_a(\tilde{\mu}) - I_a(\mu_a)| \le C'|\tilde{\mu} - \mu_a|$.
\end{assumption}
Furthermore, a similar assumption to Assumption~\ref{asm:bounded_mean_variance} is also made for this general distribution.
\begin{assumption} \label{asm:bounded_mean_variance_general}
For all $v=(v_1, v_0), v'=(v'_1, v'_0) \in \mathcal{M}^\mu$ and $a\in\{1, 0\}$, $v_a$ and $v'_a$ are mutually absolutely continuous and have the density functions.
There exist known constants $C_{\mu}, C_{\sigma^2} > 0$ such that, for all $a \in \{1,0\}$, $| \mu_a| \le  C_\mu$ and $\max\{ 1/\sigma^2_a, \sigma^2_a\} < C_{\sigma^2}$. 
\end{assumption}

Besides, owing to the asymmetry of the KL divergence, we define a local alternative hypothesis $\mathrm{Alt}(\nu) = \{\nu' \in \mathcal{M}^\mu: \mu'_1 < \mu'_0, |\mu'_1 - \mu'_0| \leq C_{\mathrm{Alt}}\Delta\}$, where $C_{\mathrm{Alt}} > 0$ is a constant. Then, for the local alternative hypothesis we can obtain the following corollary via the Fisher information, $I_a(\mu_a)$.  We show the proof later in this section. 
\begin{corollary}\label{prp:lowerbound_2arms_fisher}
Under Assumption~\ref{asm:bounded_mean_variance_general} and \ref{asm:reg_cond}, there exists a constant $C>0$ such that for each $\nu  \in \mathcal{M}^\mu$, any consistent strategy satisfies 
\begin{align*}
    \limsup_{T \to \infty} - \frac{1}{T}\log \mathbb{P}_\nu(\hat{a}_T \neq a^*(\nu)) \le \frac{\Delta^2}{2 (\sqrt{1/I_1(\mu_1)} + \sqrt{1/I_0(\mu_0)})^2} + C \Delta^3.
\end{align*}
\end{corollary}
This corollary also implies that as $\Delta \to 0$, the lower bound can be approximated by that of Gaussian bandit models. 
There are two important features of this statement. 
First, it clarifies the main term of an order $O(\Delta^2)$, and it shows that the higher-order term $C\Delta^3$ is asymptotically negligible as $\Delta \to 0$. 
Since the main terms are described by the Fisher information $I_a(\mu_a)$, it is applicable for a wide range of distributions. 

For example, the Bernoulli bandit models satisfy the assumptions, which is defined as 
$\mathcal{M}^B = \{\nu = (\mathcal{B}(\mu_1), \mathcal{B}(\mu_0)): (\mu_1, \mu_0)\in\mathbb{R}^2, \mu_1\neq \mu_0\}$, where $\mathcal{B}(\mu_a)$ denotes the Bernoulli distribution of arm $a$ with the mean $\mu_a$.  The Bernoulli bandit models satisfy Assumption~\ref{asm:reg_cond}. We can easily confirm that there exists a constant $C' > 0$ such that $|I_a(\mu'_a) - I_a(\mu_a)| = \mu'_a(1-\mu'_a) - \mu_a(1-\mu_a)\leq C'|\mu'_a - \mu_a|$. Therefore, for an alternative hypothesis $\mathrm{Alt}(\nu) = \{\nu' \in \mathcal{M}^B: \mu'_1 < \mu'_0, |\mu'_1 - \mu'_0| \leq C_{\mathrm{Alt}}\Delta\}$, the lower bound is given as the following Corollary~\ref{prp:lowerbound_2arms_bernoulli}. 
\begin{corollary}\label{prp:lowerbound_2arms_bernoulli}
Under Assumption~\ref{asm:bounded_mean_variance_general} and \ref{asm:reg_cond}, there exists a constant $C>0$ such that 
 for each $\nu  \in \mathcal{M}^B$, any consistent strategy satisfies 
\begin{align*}
    \limsup_{T \to \infty} - \frac{1}{T}\log \mathbb{P}_\nu(\hat{a}_T \neq a^*(\nu)) \le \frac{\Delta^2}{2 (\sqrt{\mu_1(1-\mu_1)} + \sqrt{\mu_0(1-\mu_0)})^2} + C \Delta^3.
\end{align*}
\end{corollary}
Here, by using the variance of $\mathcal{B}(\mu_a)$, $\sigma^2_a = \mu_a(1-\mu_a)$, the Fisher information can be written as $I_a(\mu_a) = 1/\sigma^2_a$. 

We show the proof of Corollary~\ref{prp:lowerbound_2arms_fisher} as follows. First, we show the following lemma.
\begin{lemma}
\label{lem:kl_fisher}
There exists a constant $C>0$ such that under Assumptions~\ref{asm:bounded_mean_variance_general} and \ref{asm:reg_cond}, for each $a \in \{1, 0\}$, for any $\mu_a, \mu_a' < \interior{\Theta}$ with $ |\mu_a -\mu_a'| \le \Delta$, if $\Delta$ is sufficiently small such that with the neighborhood $U(\mu_a)$ defined in Assumption~\ref{asm:reg_cond} (v),  $ \mu_a' \in U(\mu_a)$, the following holds.
\begin{align*}
\left|\KL(\nu_a', \nu_a) -  \frac{(\mu'_a - \mu_a\big)^2I_a(\mu_a)}{2} \right|\leq C\Delta^3.
\end{align*}
\end{lemma}
\begin{proof}
From the Taylor series expansion,
\begin{align}\label{eq:taylor_likelihood}
    \ell(\mu_a) -\ell(\mu_a') &= \dot{\ell}(\mu_a')(\mu_a - \mu_a') + \frac{1}{2}\ddot{\ell}(\mu_a')(\mu_a - \mu_a')^2 + \frac{1}{6}\dddot{\ell}(\tilde{\mu})(\mu_a - \mu_a')^3
\end{align}
where $\tilde{\mu}$ has the value between $\mu_a$ and $\mu_a'$.
Note that $\mathbb{E}_{\nu'}[\ell(\mu_a) - \ell(\mu_a')] = -\KL(\nu_a', \nu_a)$, $\mathbb{E}_{\nu'}[\dot{\ell}(\mu_a')] = 0$, and $\mathbb{E}_{\nu'}[\ddot{\ell}(\mu_a')] = -I_a(\mu'_a)$. Taking the expectation $\mathbb{E}_{\nu'}[\cdot]$ for both sides of \eqref{eq:taylor_likelihood}, 
\begin{align*}
    -\KL(\nu_a', \nu_a) &= -\frac{I_a(\mu'_a)}{2}(\mu_a' - \mu_a)^2 + \frac{1}{6}\Ebb_{\nu'}[\dddot{\ell}(\tilde{\mu})](\mu_a' - \mu_a)^3.
\end{align*}
From the assumption, there exist constants $c_1, c_2>0$ such that $I_a(\mu_a) - I(\mu'_a) \le c_1|\mu_a - \mu'_a| \leq c_1\Delta$ and $\Ebb_{\nu'}[\dddot{\ell}(\tilde{\mu})] \le c_2$. Therefore, we conclude the proof.
\end{proof}

Then, we show Corollary~\ref{prp:lowerbound_2arms_fisher}.
\begin{proof}
From Lemma~\ref{lem:kl_fisher}, for each $a \in \{1,0\}$,
\begin{align*}
    \KL(\nu_a', \nu_a) \leq \frac{(\mu'_a - \mu_a\big)^2I_a(\mu_a)}{2} + C\Delta^3.
\end{align*}

Then,
\begin{align*}
     \limsup_{T \to \infty} \frac{1}{T}\log \frac{1}{\Pbb_\nu(\hat{a}_T \neq 1)} & 
     \le \min_{\lambda \in \mathbb{R}} \max_{a \in \{1, 0\}}\frac{(\lambda - \mu_a)^2I_a(\mu_a)}{2} + C\Delta^3 .
\end{align*}
When the minimum over $\lambda \in \mathbb{R}$ is attained, 
\begin{align*}
 \frac{(\lambda - \mu_1)^2I_1(\mu_1)}{2 } & = \frac{(\lambda - \mu_0)^2I_0(\mu_0)}{2}
\\
    \lambda & = \frac{\mu_1 \sqrt{1/I_0(\mu_0)} + \mu_0 \sqrt{1/I_1(\mu_1)}}{\sqrt{2} \sqrt{1/I_1(\mu_1)} \sqrt{1/I_0(\mu_0)}}.
\end{align*}
Thus, we have 
\begin{align*}
    \limsup_{T \to \infty} \frac{1}{T}\log \frac{1}{\Pbb_\nu(\hat{a}_T \neq 1)} \le \frac{(\mu_1 - \mu_0)^2}{2 (\sqrt{1/I_1(\mu_1)} + \sqrt{1/I_0(\mu_0)})^2} + C\Delta^3.
\end{align*}
This concludes the proof.
\end{proof}

\section{Additional Experimental Results}
\label{sec:appdx_experiments}
Here, we show the additional results of the experiments (Section~\ref{sec:experiments}). 

\begin{figure}[htbp]
    \begin{tabular}{cc}
      \begin{minipage}[t]{0.45\hsize}
        \centering
        \includegraphics[width=70mm]{figure/Exp1-eps-converted-to.pdf}
        \caption*{Scenario~1}
      \end{minipage} &
      \begin{minipage}[t]{0.45\hsize}
        \centering
        \includegraphics[width=70mm]{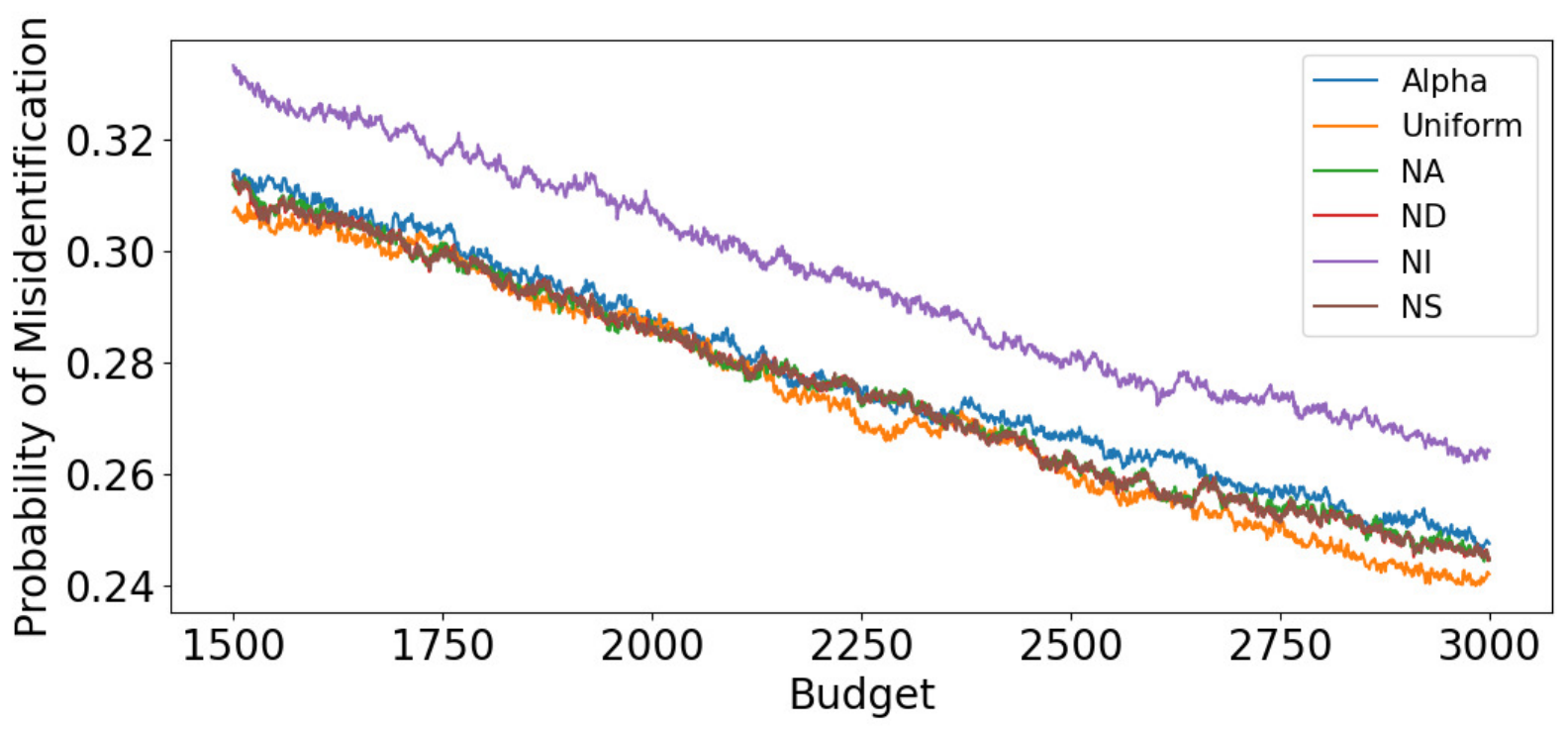}
        \caption*{Scenario~5}
      \end{minipage} \\
      \begin{minipage}[t]{0.45\hsize}
        \centering
        \includegraphics[width=70mm]{figure/Exp2-eps-converted-to.pdf}
        \caption*{Scenario~2}
      \end{minipage} &
      \begin{minipage}[t]{0.45\hsize}
        \centering
        \includegraphics[width=70mm]{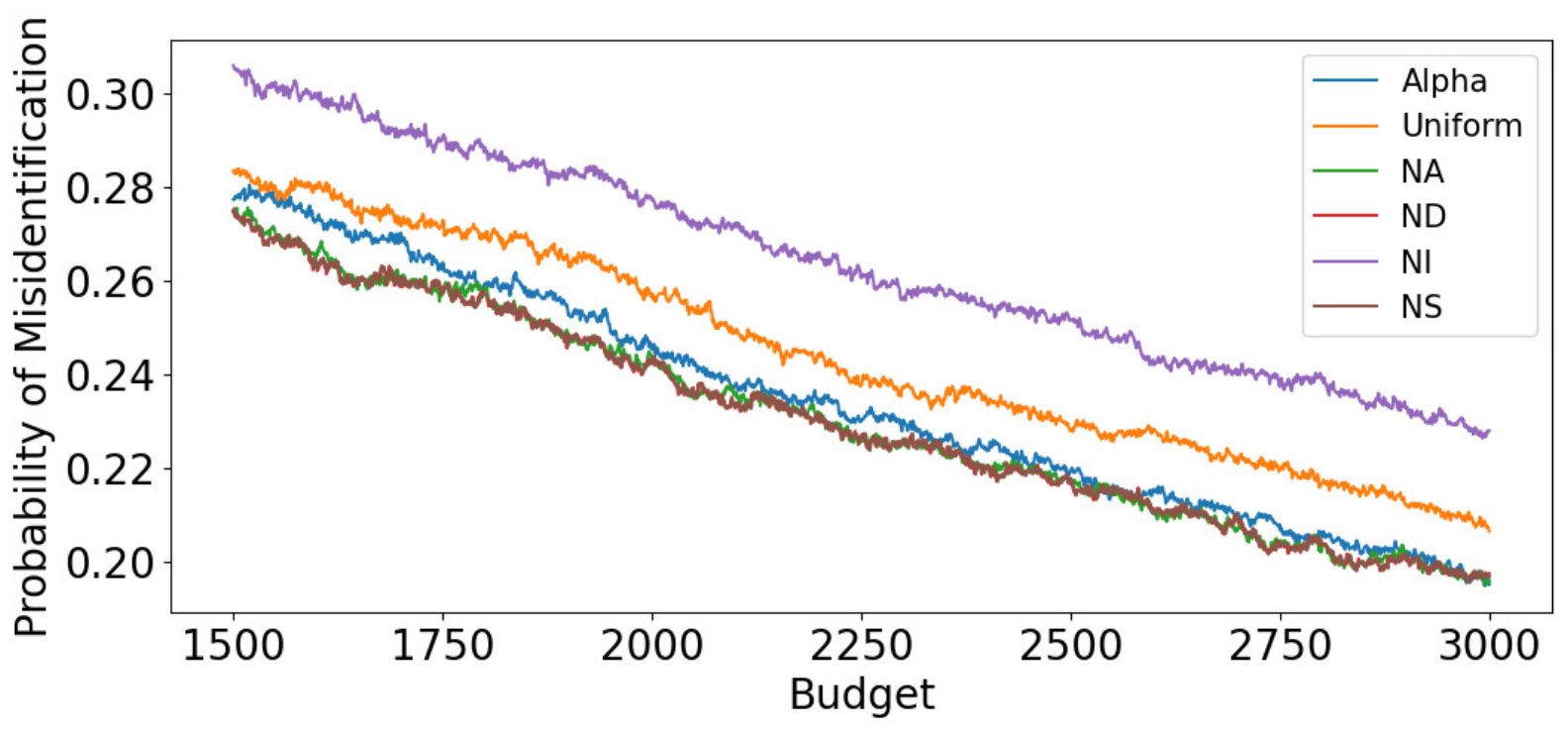}
        \caption*{Scenario~6}
      \end{minipage}\\
   
      \begin{minipage}[t]{0.45\hsize}
        \centering
        \includegraphics[width=70mm]{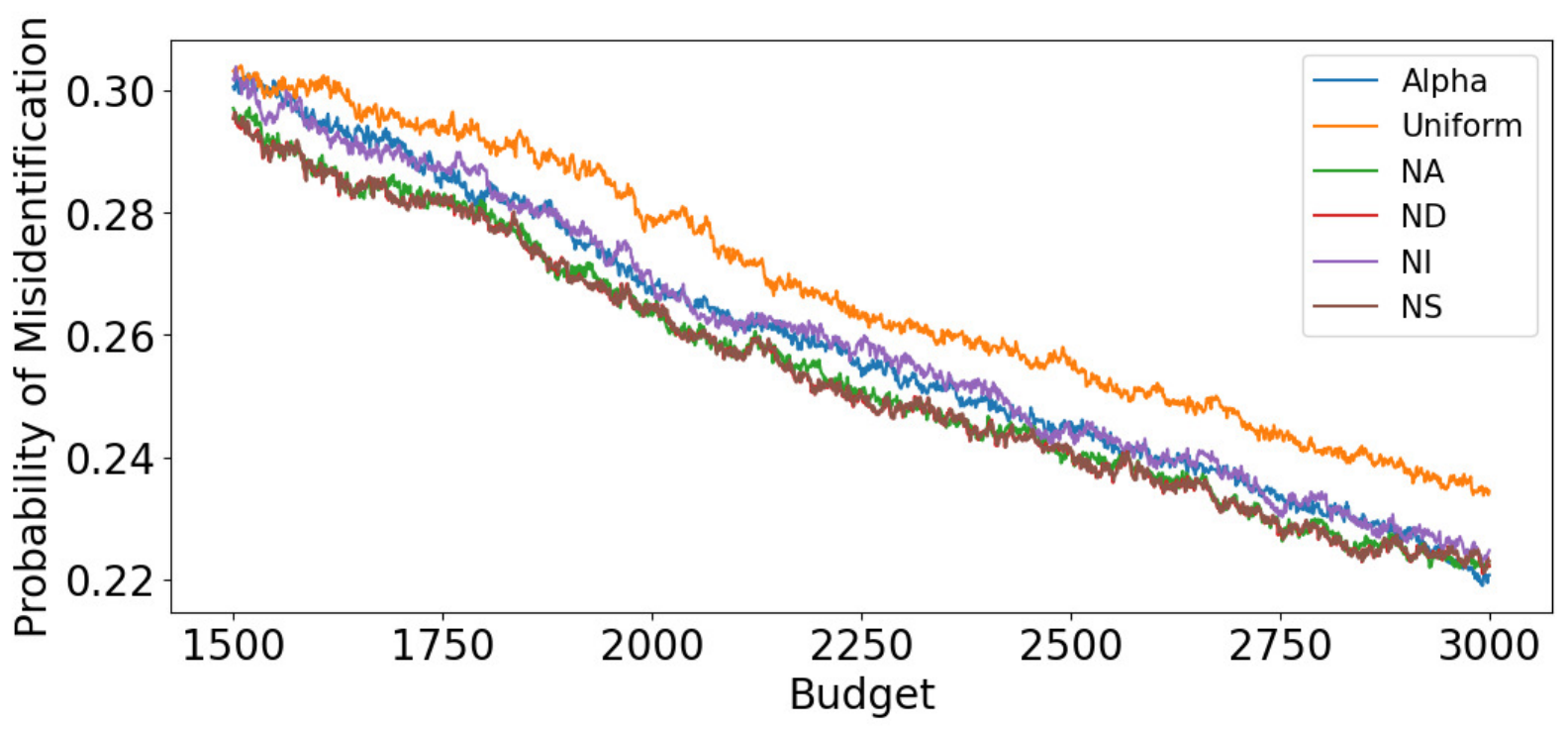}
        \caption*{Scenario~3}
      \end{minipage} &
      \begin{minipage}[t]{0.45\hsize}
        \centering
        \includegraphics[width=70mm]{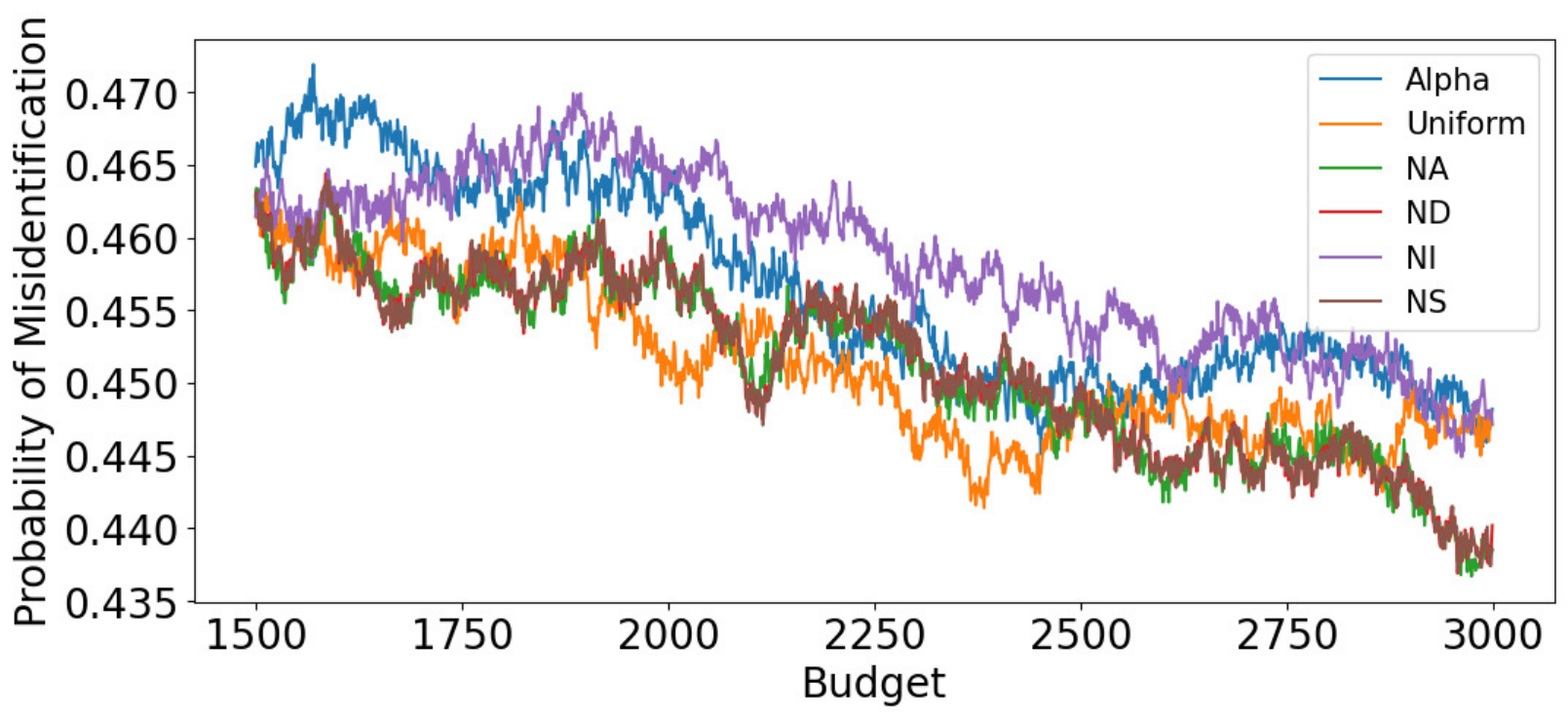}
        \caption*{Scenario~7}
      \end{minipage} \\
      
      \begin{minipage}[t]{0.45\hsize}
        \centering
        \includegraphics[width=70mm]{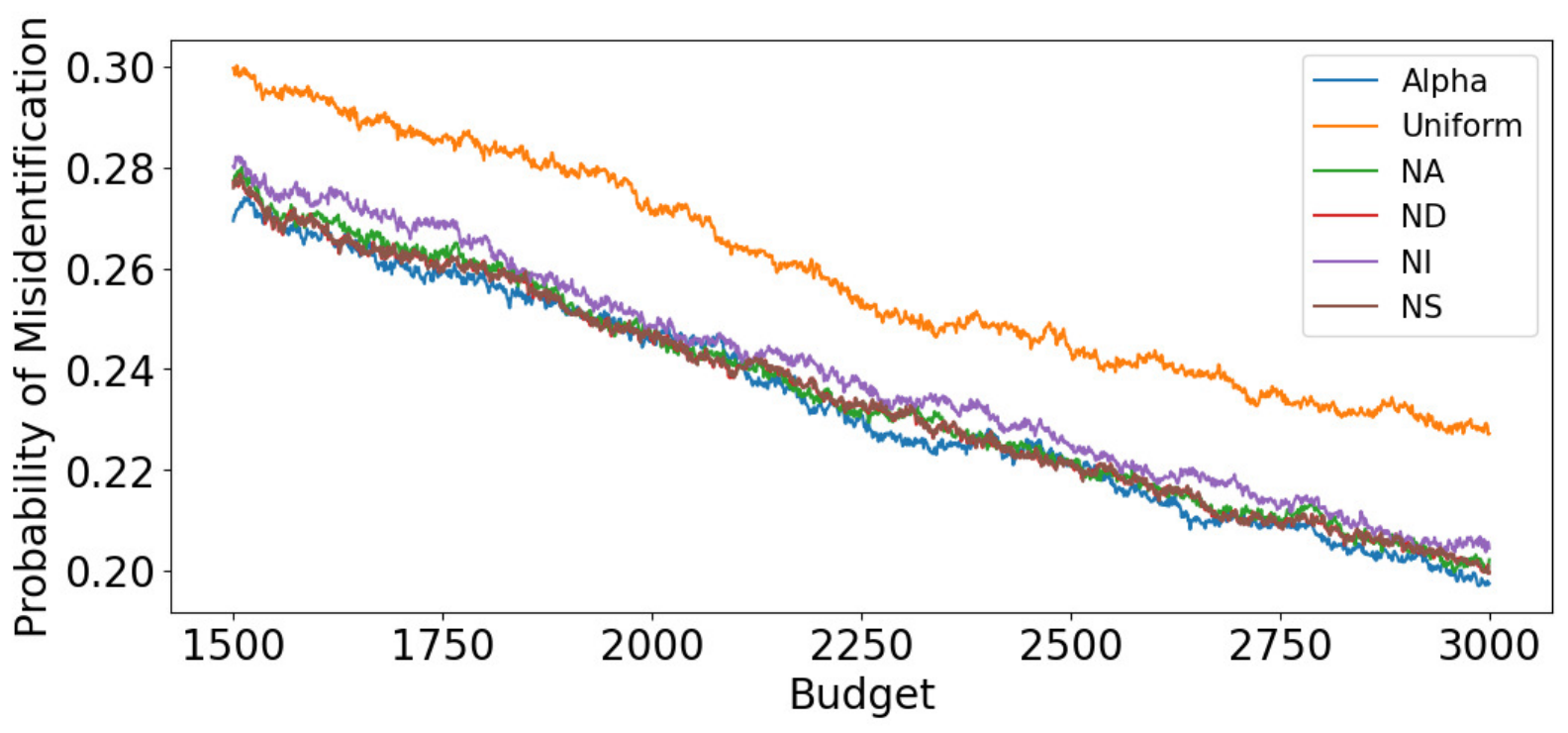}
        \caption*{Scenario~4}
      \end{minipage} &
      \begin{minipage}[t]{0.45\hsize}
        \centering
        \includegraphics[width=70mm]{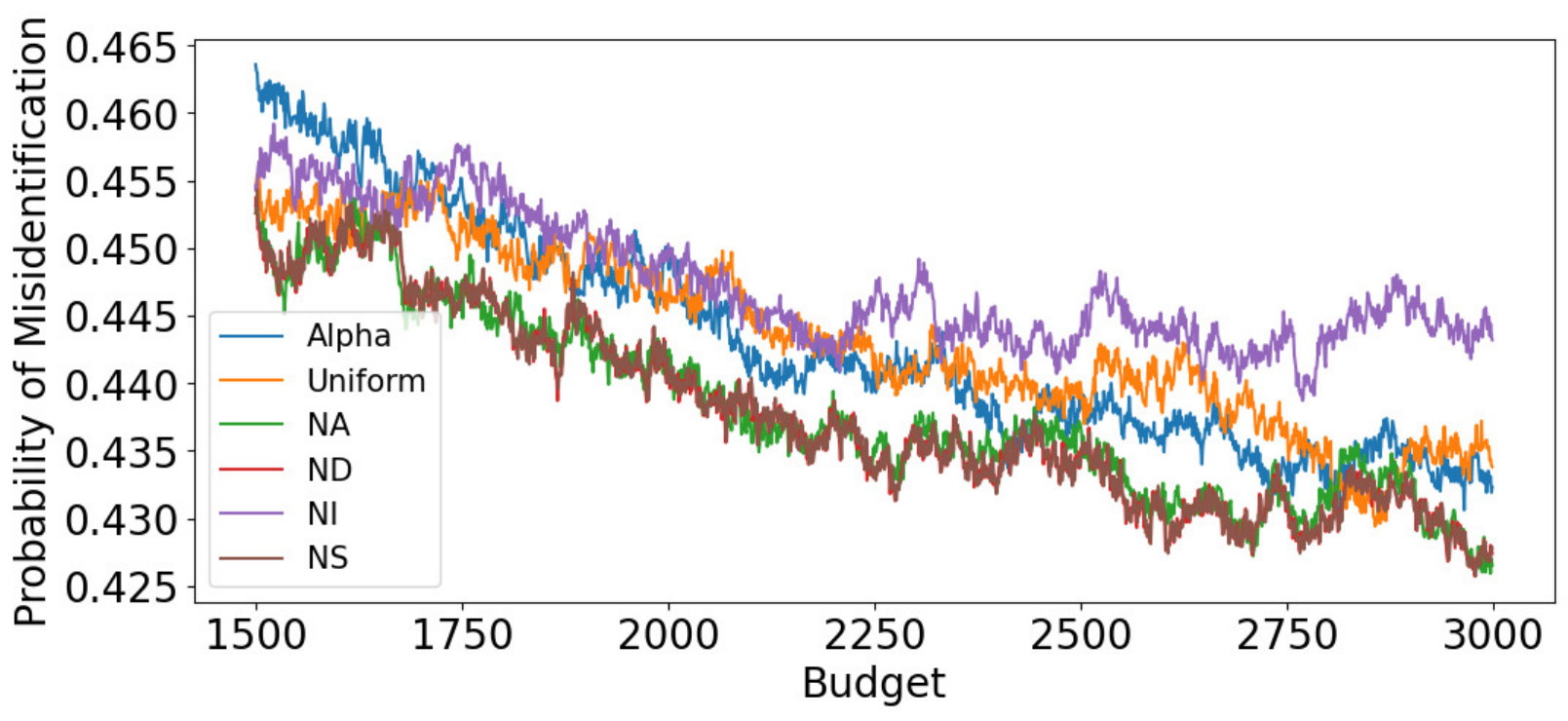}
        \caption*{Scenario~8}
      \end{minipage}
    \end{tabular}
    \caption{Results of the two-armed Gaussian bandits. We compute the empirical probability of misidentification. Note that the result of ND is overlapped with that of the NS.}
\label{fig:synthetic_results2}
  \end{figure}

\end{document}